\documentclass{article}

\usepackage{druv_macros}  %
\usepackage{packages_macros}  %

\usepackage[accepted]{icml2025}

\usepackage[textsize=tiny]{todonotes}
\usepackage[percent]{overpic}

\icmltitlerunning{Simplifying DINO via Coding Rate Regularization}

\begin{document}

\twocolumn[
\icmltitle{Simplifying DINO via Coding Rate Regularization}

\icmlsetsymbol{equal}{*}

\begin{icmlauthorlist}
\icmlauthor{Ziyang Wu}{ucb}
\icmlauthor{Jingyuan Zhang}{tsc}
\icmlauthor{Druv Pai}{ucb}
\icmlauthor{Xudong Wang}{ucb}\\
\icmlauthor{Chandan Singh}{ms}
\icmlauthor{Jianwei Yang}{ms}
\icmlauthor{Jianfeng Gao}{ms}
\icmlauthor{Yi Ma}{ucb,tsc,hku}
\end{icmlauthorlist}

\icmlaffiliation{ucb}{UC Berkeley}
\icmlaffiliation{ms}{Microsoft Research}
\icmlaffiliation{tsc}{TranscEngram}
\icmlaffiliation{hku}{HKU}

\icmlcorrespondingauthor{Ziyang Wu}{zywu@berkeley.edu}

\icmlkeywords{Self-supervised learning, DINO}

\vskip 0.3in
]

\printAffiliationsAndNotice{}  %

\begin{abstract}
\dino{} and \dino{}v2 are two model families being widely used to learn representations from unlabeled imagery data at large scales. Their learned representations often enable state-of-the-art performance for downstream tasks, such as image classification and segmentation. However, they employ many empirically motivated design choices and their training pipelines are highly complex and unstable --- many hyperparameters need to be carefully tuned to ensure that the representations do not collapse --- which poses considerable difficulty to improving them or adapting them to new domains. In this work, we posit that we can remove most such-motivated idiosyncrasies in the pre-training pipelines, and only need to add an explicit coding rate term in the loss function to avoid collapse of the representations. As a result, we obtain highly simplified variants of the \dino{} and \dino{}v2 which we call \simdino{} and \simdino{}v2, respectively. Remarkably, these simplified models are more robust to different design choices, such as network architecture and hyperparameters, and they learn even higher-quality representations, measured by performance on downstream tasks, offering a Pareto improvement over the corresponding \dino{} and \dino{}v2 models. This work highlights the potential of using simplifying design principles to improve the empirical practice of deep learning.
\end{abstract}

\section{Introduction}

Self-supervised learning (SSL) is the toolkit of choice to learn representations for large datasets of unlabeled images \citep{hadsell2006dimensionality,oord2018representation,wu2018unsupervised,grill2020bootstrap,he2020momentum,bardes2021vicreg,chen2021exploring,caron2021emerging,zhou2021ibot,he2022masked,assran2023self,oquab2023dinov2}, captioned images \citep{radford2021learning}, videos \cite{feichtenhofer2022masked}, and text \citep{radford2018improving,devlin2018bert,radford2019language,brown2020language}, among other modalities. In the context of image SSL, there are two main approaches: \textit{reconstructive} \citep{he2022masked}, where the goal is to reconstruct some function of the true image data from a ``view'', i.e., corruption or augmentation, and \textit{contrastive} \citep{hadsell2006dimensionality}, where the goal is, for each image, to have the features of different views of the image all be close, and features of views of different images be far. 

Within contrastive SSL, a key challenge lies in preventing \textit{representation collapse}, where models learn trivial solutions that map all inputs to the same output. One common approach to address this is through the use of \textit{negative samples}, which explicitly encourages representations of different images to be dissimilar. Thus far, the success of using negative samples depends on having a large batch size \citep{wu2018unsupervised,he2020momentum}, which poses computational challenges at scale. Methods which attempt to avoid this bottleneck by using negative samples in more implicit and indirect ways to avoid collapse \citep{caron2021emerging} can cope with smaller batch sizes, but often require training pipelines with many components and hyperparameters carefully tuned to avoid collapse, making them difficult to train.

The state-of-the-art for image SSL is generally considered to be the \dino{}v2 model family \citep{oquab2023dinov2}. It is built on the \dino{} model family \citep{caron2021emerging}. Both classes of models are trained using contrastive SSL and thus run into the representation collapse issue. While \dino{}v2 explicitly and directly uses negative samples to avoid collapse, it inherits much of its training pipeline from \dino{}, which uses negative samples more indirectly. As such, \textit{both} model families' training pipelines are highly complex and unstable, requiring many tweaks and careful hyperparameter selection in order for the training to converge for a given architecture. Despite this capriciousness, the trained models' representations are highly useful for downstream tasks, and are widely used \cite{baharoon2023towards,wei2024stronger}.

\begin{figure*}[t]
    \begin{subfigure}[t]{0.55\textwidth}
        (a): \dino{}
        
        \includegraphics[width=\textwidth]{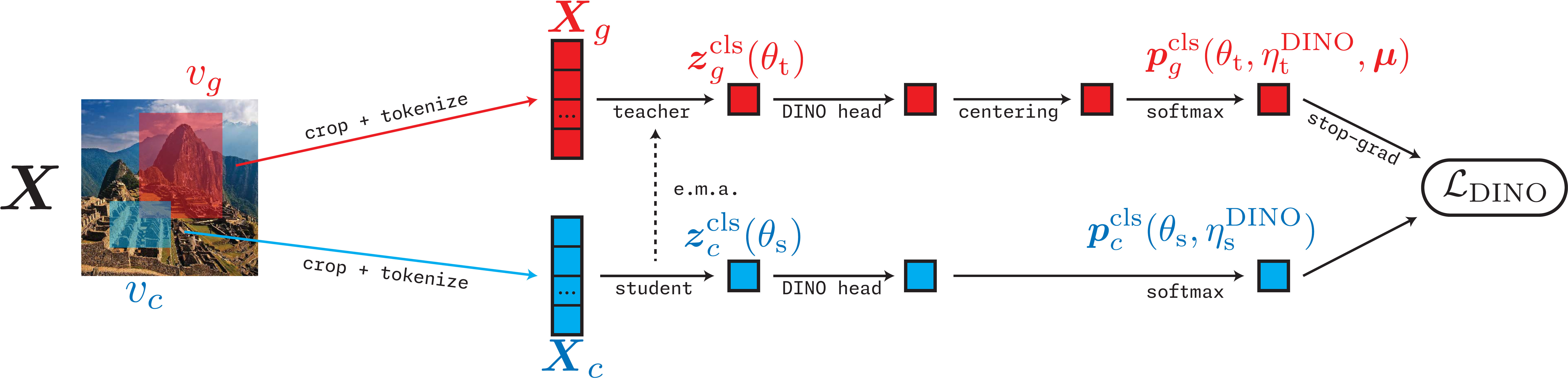}
    \end{subfigure}
    \hfill 
    \begin{subfigure}[t]{0.4\textwidth}
        (b): \simdino{}

        \includegraphics[width=\textwidth]{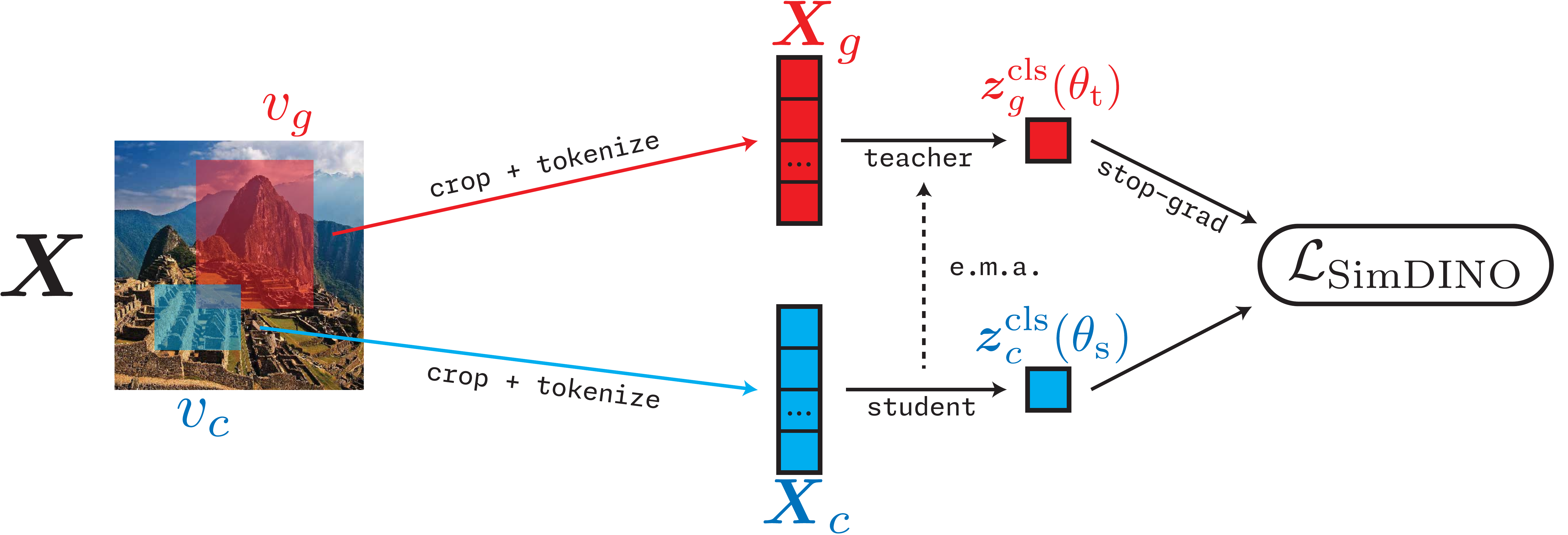}
    \end{subfigure}

    \vspace{1em}
    
    \begin{subfigure}[t]{0.55\textwidth}
        (c): \dino{}v2
        
        \includegraphics[width=\textwidth]{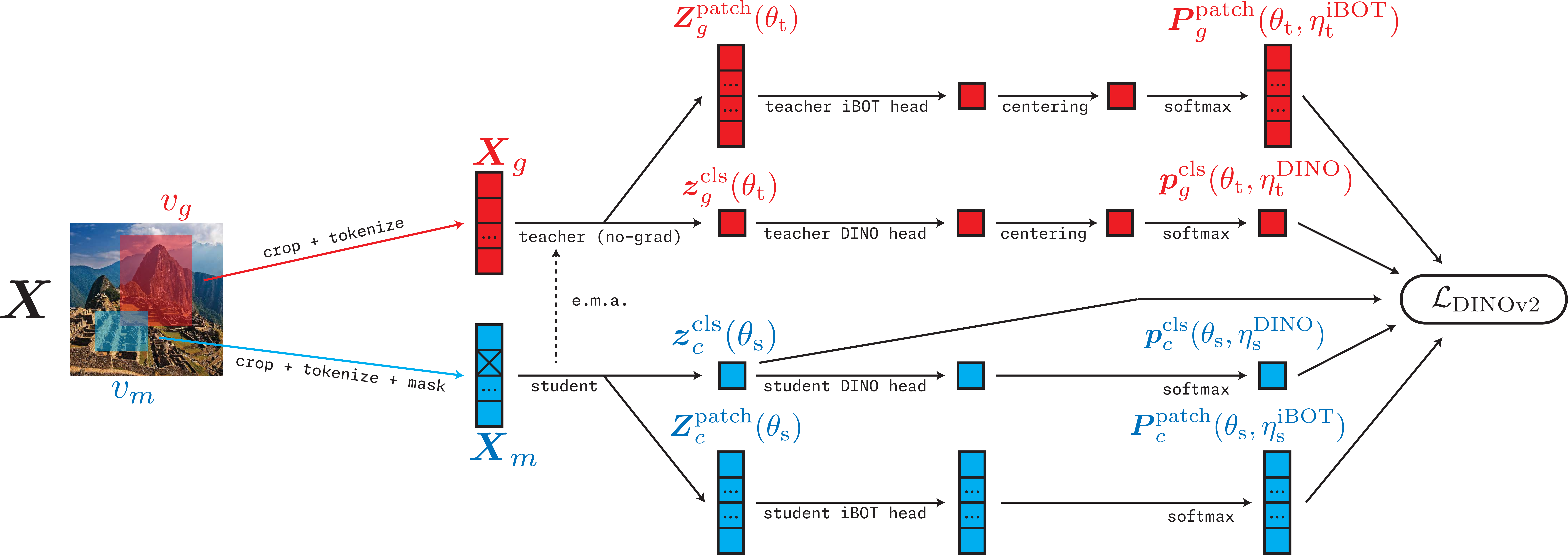}
    \end{subfigure}
    \hfill 
    \begin{subfigure}[t]{0.4\textwidth}
        (d): \simdino{}v2
        
        \includegraphics[width=\textwidth]{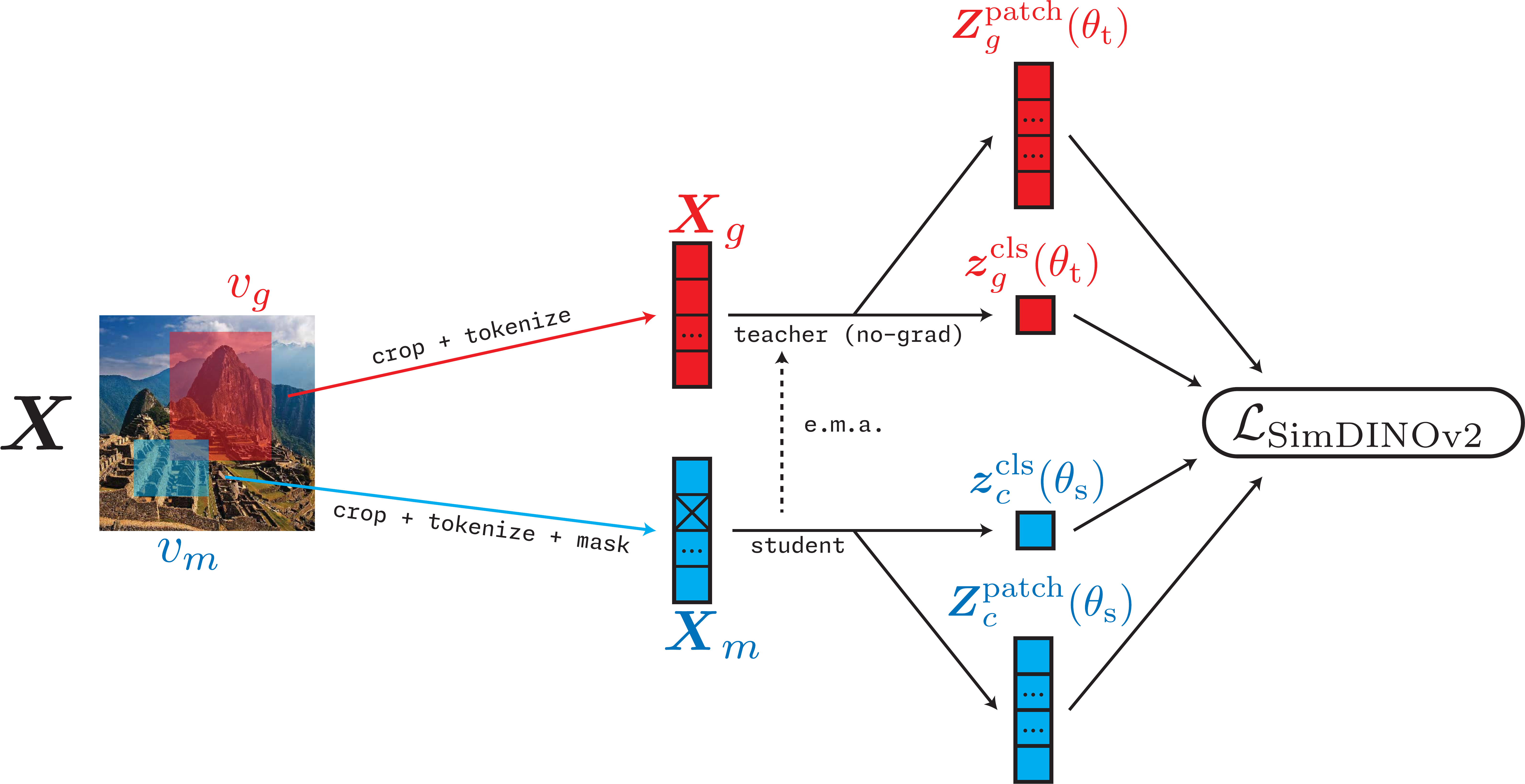}
    \end{subfigure}

    \caption{
    \small
    \textbf{The \dino{} and \dino v2 pipelines are substantially simplified to the respective  \simdino{} and \simdino v2 pipelines.}
    (a) In the \dino{} pipeline, an input image is turned into patches. Then a global view \(v_{\globalview}\) and a local view \(v_{\crop}\) are randomly sampled.
    The global view is pushed through the teacher encoder, while the other view is through the student encoder.
    (b) The \simdino{} pipeline removes the need for expensive post-processing operations present in \dino{}, such as a dimension-increasing linear layer and a high-dimensional softmax.
    (c) The \dino{}v2 pipeline adds masking (here masked patches are denoted by \(\times\)) and an additional loss on image patch features to the \dino{} pipeline.
    (d) The \simdino{}v2 training operates directly on the learned representations, simplifying the pipeline.
    \vspace{-2mm}
    }
    \label{fig:dino_pipeline}
\end{figure*}

\paragraph{Our contributions.}
In this work, we remove many tweaks and hyperparameters from the \dino{} and \dino{}v2 training pipelines, replacing them with a term in the objective which explicitly uses negative samples. We show empirically that this term, which involves the \textit{total coding rate} regularizer \citep{ma2007segmentation,yu2020learning,li2022neural}, enables much more simple, robust, and computationally efficient training pipelines, as shown in Figure \ref{fig:dino_pipeline}. We show that the resulting models, named \simdino{} and \simdino{}v2, learn representations that achieve even higher state-of-the-art performance as those learned by \dino{} and \dino{}v2 across a variety of downstream tasks. Our work underscores the value of understanding and simplifying pipelines to improve performance in vision SSL.

\vspace{-2mm}
\paragraph{Notation.} Let \(C, H, W, D, N, d \geq 1\) be positive integers. Let the space of finite sequences of vectors in \(\R^{D}\) be denoted as \(\R^{D \times *} = \bigcup_{T = 1}^{\infty}\R^{D \times T}\). Our data will be images \(\vX \in \R^{C \times H \times W}\). We consider different augmentations, or \textit{views}, of the input data \(\vX\), such as rotations or crops; we can represent a view as a function \(v \colon \R^{C \times H \times W} \to \R^{D \times N_{v}}\) where \(N_{v}\) is the number of tokens in the view. 

Let \(\Sphere_{d - 1} \subseteq \R^{d}\) be the \((d-1)\)-dimensional \(\ell^{2}\)-sphere. For the purpose of representation learning, we will consider an \textit{encoder} neural network parameterized by weights \(\theta \in \Theta\), as a function \(f_{\theta} \colon \R^{D \times *} \to \Sphere_{d-1} \times \Sphere_{d - 1}^{N}\). We factor \(f_{\theta} = (f_{\theta}^{\cls}, f_{\theta}^{\patch})\) where \(f_{\theta}^{\cls} \colon \R^{D \times *} \to \Sphere^{d-1}\) outputs the so-called \textit{class token feature} (i.e., an aggregate representation of the input data) and \(f_{\theta}^{\patch} \colon \R^{D \times *} \to \Sphere_{d-1}^{N}\) outputs the \textit{patch tokens' features} (i.e., a patch-based representation of the input data). The network is implemented by a Vision Transformer \citep{dosovitskiy2020image,touvron2021training} with appended multi-layer perceptrons (MLPs) to post-process each feature followed by \(\ell^{2}\)-normalizations.

\section{Methods: Simplifying \dino{} and \dino{}v2}

\subsection{Recap of the Original \dino{} Pipeline}

The goal of \dino{} is to learn an aggregate representation of the input image which contains information about large-scale semantics of the input (e.g., the locations and properties of different objects in the image). They do this via a pre-training pipeline \citep{caron2021emerging} which is depicted in \Cref{fig:dino_pipeline}(a), and we also describe it throughout this section. The main idea is to take multiple \textit{views} (i.e., different crops) of the data, and ensure that the features generated by these views are consistent with each other (in a sense which will be made precise shortly) as much as possible. If the views each contain a salient part of the input such as a central object, the feature corresponding to any view would then contain information about this central object. The end goal is that the feature of any large-enough view contains information about all relevant objects in the input image, which can then be extracted for use in downstream tasks such as image classification or image segmentation.

In the rest of the section, we will discuss the pre-training pipeline. As is common in contrastive SSL, the \dino{} framework uses two networks: a so-called \textit{teacher} network parameterized by \(\theta_{\teacher} \in \Theta\), and a so-called \textit{student} network parameterized by \(\theta_{\student} \in \Theta\). During pre-training, the loss will encourage the student's representation to align with the teacher's representation, even as the teacher is simultaneously updated using student weights; this is \textit{self-distillation}, and can be viewed as an optimization strategy or even implicitly regularizing the objective \citep{chen2021exploring}.

During the pipeline, we process each image \(\vX\) in the following way. First, we sample at random a view, or crop, \(v_{\crop}\), independently of \(\vX\); the view can \textit{either} be a ``global view'' (i.e., a large crop) or a ``local view'' (i.e., small crop), selected randomly. We denote \(\vX_{\crop} := v_{\crop}(\vX)\). In addition, we sample a global view \(v_{\globalview}\) independently of \(\vX\) and \(v_{\crop}\), and denote \(\vX_{\globalview} := v_{\globalview}(\vX)\).\footnote{More precisely, let \(\crop\) be a random vector containing the boundaries of the crop, so that \(v_{\crop}\) crops exactly the region supplied by \(\crop\). Analogous notation can be defined for \(\globalview\) and \(v_{\globalview}\).} Views are implemented in the same way as in \dino{}; they are formally described in \Cref{sec:views_formal} for the sake of completeness. 

The first (local or global) view \(\vX_{\crop}\) is fed to the student network\footnote{Note that the parameters \(\theta_{\student}\) and \(\theta_{\teacher}\) each contain a positional encoding over all patches; when a view is fed through the network, it receives a interpolated positional encoding corresponding to the tokens' length.} \(f_{\theta_{\student}}\) to get an aggregate representation \(\vz_{\crop}^{\cls}(\theta_{\student})\), while the global view \(\vX_{\globalview}\) is fed to the teacher network \(f_{\theta_{\teacher}}\) to get \(\vz_{\globalview}^{\cls}(\theta_{\teacher})\), i.e.:
\begin{equation}
    \vz_{\crop}^{\cls}(\theta_{\student}) := f_{\theta_{\student}}^{\cls}(\vX_{\crop}), \qquad \vz_{\globalview}^{\cls}(\theta_{\teacher}) := f_{\theta_{\teacher}}^{\cls}(\vX_{\globalview}).
\end{equation} 
Now, it is certainly possible to directly compare and evaluate these features. However, \dino{} adds post-processing steps, arguing that they improve performance and prevent collapse:
\begin{itemize}
    \item They add weight-normalized linear layers \citep{salimans2016weight} \(h_{\eta_{\student}^{\dino}}, h_{\eta_{\teacher}}^{\dino} \colon \R^{d} \to \R^{m}\) where \(m \gg d\), called the ``\dino{} heads'' and parameterized by \(\eta_{\student}^{\dino}, \eta_{\teacher}^{\dino}\), appended to the end of the student and teacher networks respectively.
    \item They center the teacher-computed features using a learned vector \(\vmu \in \R^{m}\).
    \item They take a temperature-weighted softmax of both features to compute probability vectors in \(\R^{m}\), sometimes called \textit{prototype scores}, which they then can compare using cross-entropy.
\end{itemize}
Mathematically, the post-processing steps to get probability vectors for each view are as follows:
\begin{align}
    &\vp_{\crop}^{\cls}(\theta_{\student}, \eta_{\student}^{\dino}) := \softmax(h_{\eta_{\student}^{\dino}}(\vz_{\crop}^{\cls}(\theta_{\student}))/\tau), \\
    \label{eq:dino_teacher_probability_vector}
    &\scalebox{0.8875}{\(\vp_{\globalview}^{\cls}(\theta_{\teacher}, \eta_{\teacher}^{\dino}, \vmu)  := \softmax([h_{\eta_{\teacher}^{\dino}}(\vz_{\globalview}^{\cls}(\theta_{\teacher})) - \vmu]/\tau)\)},
\end{align}
where \(\tau > 0\) is the temperature parameter. Finally, the loss (to be minimized) encourages \(\vp_{\crop}^{\cls}\) and \(\vp_{\globalview}^{\cls}\) to be close together using a symmetrized cross-entropy-based functional \(d_{\CE}\), which effectively distills the teacher into the student by aligning the predicted outputs:
\begin{align}\label{eq:DINO_loss}
    &\cL_{\dino}(\theta_{\student}, \theta_{\teacher}, \eta_{\student}^{\dino}, \eta_{\teacher}^{\dino}, \vmu) \\
    &:= \Ex[d_{\CE}(\vp_{\crop}^{\cls}(\theta_{\student}, \eta_{\student}^{\dino}), \vp_{\globalview}^{\cls}(\theta_{\teacher}, \eta_{\teacher}^{\dino}, \vmu))] \nonumber
\end{align}
where the expectation is over \(\vX\), the (local or global) view \(v_{\crop}\), and the global view \(v_{\globalview}\) sampled i.i.d., and the function \(d_{\CE}\) is defined via the cross-entropy as 
\begin{align}\label{eq:cross_entropy_divergence}
    d_{\CE}(\vp, \vq) := \frac{1}{2}\bp{\CE(\vp, \vq) + \CE(\vq, \vp)},\\\CE(\vp, \vq) := -\sum_{i = 1}^{m}p_{i}\log q_{i}.
\end{align}
When training, \dino{} estimates the expectation in \eqref{eq:DINO_loss} by a stratified plug-in estimator over a batch of sample images. That is, to estimate the expectation, we condition on \(\vX\) then estimate the conditional expectation \(\Ex[d_{\CE}(\cdot, \cdot) \given \vX]\) via plug-in using several different global views (usually two global views, which play the role of the arbitrary view \(v_{\crop}\) and the global view \(v_{\globalview}\)) and several different local views, and finally average over \(\vX\) to obtain the estimate. 
The optimization of this estimated loss, too, is done in an ad-hoc way; while all four parameters \(\theta_{\student}, \theta_{\teacher}, \eta^{\dino}, \vmu\) are updated at each iteration, they update in different ways:
\begin{itemize}
    \item The student parameters \(\theta_{\student}\) and \(\eta_{\student}^{\dino}\) are updated via an iteration of a stochastic gradient descent (SGD)-type algorithm, such as Adam, on the loss \eqref{eq:DINO_loss}. The backpropagation for the loss gradient is computed assuming the teacher parameters \(\theta_{\teacher}, \eta_{\teacher}^{\dino}\), and \(\vmu\) are ``frozen'' or constants (i.e., ``stop-gradient'').
    \item The teacher parameters \(\theta_{\teacher}\), \(\eta_{\teacher}^{\dino}\), and \(\vmu\) are updated via exponentially moving averages (EMAs) of the student weights \(\theta_{\student}\), the student \dino{} head \(\eta_{\student}^{\dino}\), and the average output of teacher the \dino{} head \(\Ex[h_{\eta_{\teacher}^{\dino}}(\vz_{\crop}^{\cls}(\theta_{\teacher}))]\) (in practice taken over a minibatch), respectively. Formally, for decay parameters \(\lambda, \nu \in [0, 1]\), at each iteration we compute  \(\theta_{\teacher} \gets \lambda \theta_{\teacher} + (1 - \lambda)\theta_{\student}\), \(\eta_{\teacher}^{\dino} \gets \lambda \eta_{\teacher}^{\dino} + (1 - \lambda)\eta_{\student}^{\dino}\), and \(\vmu \gets \nu \vmu + (1 - \nu)\Ex[h_{\eta_{\teacher}^{\dino}}(\vz_{\globalview}^{\cls}(\theta_{\teacher}))]\).
\end{itemize}
The decay parameters \(\lambda, \nu\) and the temperature parameter \(\tau\) change along the optimization trajectory, and their schedules are design decisions which impact convergence. 

As previously mentioned, many of the ad-hoc methods and choices described above are due to a tension: a \textit{trivial} solution to optimizing \eqref{eq:DINO_loss} is to enforce that \(f_{\theta_{\student}}\) and \(f_{\theta_{\teacher}}\) \textit{collapse}, i.e., become or approximate the constant function, which map each local and global view to the same feature \(\vz\) or even to the same probability vector \(\vp\). To explain why \dino{} does not collapse, we wish to highlight the centering operation in \eqref{eq:dino_teacher_probability_vector}, which computes batch statistics during its EMA update, hence using negative samples and implicitly pushing different samples' features apart, even though the precise conceptual mechanism by which this occurs is not clear and involves a careful interaction between the centering vector and temperature scaling \citep{caron2021emerging}. Indeed, \citet{caron2021emerging} shows that collapsed solutions are common without very carefully tuning the EMA schedule and temperature schedule, and arguing that the remaining hyperparameters and choices would severely degrade the performance if perturbed. A more in-depth discussion of the tension, and the added complexity required to train a model in spite of it, is in \Cref{sub:dino_complexity}. As we will see, if this tension is alleviated in an alternative way, many hyperparameters can be removed and the rest can be changed robustly.

\subsection{From \dino{} to \simdino{}}

To go from \dino{} to \simdino{}, we ask the question:
\begin{quote}
    \centering
    \textit{
        Can we directly compare \(\vz_{\crop}^{\cls}\) and \(\vz_{\globalview}^{\cls}\)?
    }
\end{quote}
If we could do this, then we could avoid the large \dino{} head, the centering operation, the softmaxes, and the cross-entropy based loss. However, the mechanism in \dino{} for avoiding representation collapse via negative samples would therefore be removed. Thus, we have a second question:
\begin{quote}
    \centering 
    \textit{
        Can we efficiently use the negative samples' features explicitly to enforce non-collapse?
    }
\end{quote}

For the first question, we argue that the most simple \textit{squared Euclidean distance}, namely
\begin{equation}
    d_{\ell^{2}}(\vx, \vy) := \frac{1}{2}\norm{\vx - \vy}_{2}^{2}
\end{equation}
 works at least as well as the cross-entropy-based functional \eqref{eq:cross_entropy_divergence} applied to an affine transformation of the features, as in \eqref{eq:DINO_loss}. For the second question, we argue that we may directly penalize the covariance of the features in order to avoid collapse, as follows. For a hyperparameter \(\eps > 0\), the (total) coding rate \citep{ma2007segmentation,yu2020learning,li2022neural} of a symmetric positive semidefinite matrix \(\vGamma \in \R^{d \times d}\) is 
\begin{equation}
    R_{\eps}(\vGamma) := \frac{1}{2}\logdet\rp{\vI + \frac{d}{\eps^{2}}\vGamma},
\end{equation}
In words, \(R_{\eps}\) is an approximation to the rate distortion of a Gaussian random variable with covariance \(\vGamma\) (and this approximation is perfect in the limit \(\eps \to 0\)). More concretely, it is a measure of size of the covariance, even if the underlying variables are non-Gaussian. Thus one way to ensure non-collapse is to add \(-R_{\eps}(\Cov[\vz_{v}^{\cls}(\theta_{\student})])\) as a regularizer in the objective, leading to the loss
\begin{align}\label{eq:simdino_loss}
    \cL_{\simdino}(\theta_{\student}, \theta_{\teacher}) := &\Ex[d_{\ell^{2}}(\vz_{\crop}^{\cls}(\theta_{\student}), \vz_{\globalview}^{\cls}(\theta_{\teacher}))] \\
    &- \gamma R_{\eps}(\Cov[\vz_{\crop}^{\cls}(\theta_{\student})]). \nonumber
\end{align}
where \(\gamma > 0\) is a hyperparameter. Note that \(d_{\ell^{2}}(\vz_{\crop}^{\cls}, \vz_{\globalview}^{\cls}) = -(\vz_{\crop}^{\cls})^{\top}\vz_{\globalview}^{\cls}\) since \(\vz_{\crop}^{\cls}, \vz_{\globalview}^{\cls} \in \Sphere_{d - 1}\).

When training, similar to \dino{}, we estimate the expectation and covariance in \eqref{eq:simdino_loss} by a type of plug-in estimator. Namely, the expectation is estimated similar to \dino{}, just using \(d_{\ell^{2}}\) instead of \(d_{\CE}\). To estimate the coding rate, we sub-sample several \(\vz_{\crop}^{\cls}(\theta_{\student})\) over both \(\vX\) and \(v_{\crop}\),\footnote{In practice, we only let \(v_{\crop}\) be a global view for efficiency, and offer the following heuristic justification. If the expected similarity term in \eqref{eq:simdino_loss} is large, then there is little difference between the features of local and global views. Hence \(\Cov[\vz_{\crop}^{\cls}(\theta_{\student})] \approx \Cov[\vz_{\globalview^{\prime}}^{\cls}(\theta_{\student})]\), where \(v_{\globalview^{\prime}}\) is a randomly sampled global view.} estimate \(\Cov[\vz_{\crop}^{\cls}(\theta_{\student})]\) on that sub-sample via plug-in, estimate \(R_{\eps}\) of the population covariance by calculating it on the sample covariance, then average the estimates over all sub-samples. We conjecture that the latter estimator has lower variance compared to the naive plug-in estimator for \(\Cov[\vz_{\crop}^{\cls}(\theta_{\student})]\) as it is similar to variance-reduction methods in statistics \citep{kahn1953methods}, which we hypothesize might be a factor as to why \simdino{} can handle a smaller batch size than other contrastive SSL methods that explicitly use negative samples but avoid collapse using higher-variance or more implicit regularizers.

The overall pipeline is shown in \Cref{fig:dino_pipeline}(b). Note that it is much simpler than \dino{}. We provide pseudocode for the training pipeline in \Cref{alg:simdino_training_pipeline} in \Cref{app:pipeline_pseudocode}.

After training, we discard student weights and use teacher weights for evaluation.

\subsection{From \dino{}v2 to \simdino{}v2}

The pipeline of the \dino{}v2 framework \citep{oquab2023dinov2}, as shown in \Cref{fig:dino_pipeline}(c), is built upon the \dino{} pipeline, and has two main goals: first, learn an \textit{aggregate} representation which contains large-scale semantics of the input (i.e., the goal of \dino{}); second, learn \textit{patch-based} representations which have fine-grained semantic information about each patch and its local neighborhood. The main new ideas to achieve this, drawn from the \ibot{} pipeline \citep{zhou2021ibot}, are that \textit{the input to the student has some masked patches}, and that \textit{the loss also computes similarity of the patch-based features}. To see why this works, consider if some patches are masked, and the model is able to predict masked patches using their unmasked neighbors, then from each patch the model can extract strong information about the semantics of nearby patches, which is an idea similar in spirit to masked autoencoding \citep{he2022masked}. Thus, these two ideas from \ibot{} would furnish our model with informative patch-based representations.

We now discuss the DINOv2 pipeline, before discussing our modifications. Formally, starting with tokenized images \(\vX \in \R^{D \times N}\), we take a view \(v_{\maskcrop}\) sampled at random; the view can be a global view or a local view, \textit{but it also replaces a fraction \(\alpha \in [0, 1]\) of the tokens in the view with a learnable mask token \(\vx_{\mathrm{mask}}\)} (as in \citep{he2022masked}, the mask token is shared across all views). We denote \(\vX_{\maskcrop} := v_{\maskcrop}(\vX)\). We also take a global view \(v_{\globalview}\) \textit{without} masking, independently of \(v_{\maskcrop}\) and \(\vX\), and denote \(\vX_{\globalview} := v_{\globalview}(\vX)\).

Now that we have this setup, we do similar operations to \dino{} pipeline, with some changes:
\begin{itemize}
    \item There are additional ``\ibot{} heads'' for the student and teacher, processing the patch-based features column-wise (i.e., patch-wise), with weights \(\eta_{\student}^{\ibot}, \eta_{\teacher}^{\ibot}\) (cf the ``\dino{} head'' with weights \(\eta_{\student}^{\dino}, \eta_{\teacher}^{\ibot}\)).
    \item The centering operation on teacher-output features is performed on both the aggregate features and (column-wise) on the patch-wise features. 
    \item The centering operation uses three iterations of the Sinkhorn-Knopp algorithm \citep{cuturi2013sinkhorn,caron2020unsupervised}, denoted below by \(\textsc{SKC}\), instead of an EMA, and is parameter-free but more expensive than simple subtraction. Note that the Sinkhorn-Knopp algorithm uses features from all images in each minibatch.
\end{itemize}
Let \(\vz^{i} \in \R^{d}\) be the \(i^{\mathrm{th}}\) column of \(\vZ^{\patch}\) (and similar for \(\vp^{i} \to \vP^{\patch}\)). Then, formally we have (where \(1 \leq i \leq N\))
\begin{align}
    &(\vz_{\maskcrop}^{\cls}(\theta_{\student}), \vZ_{\maskcrop}^{\patch}(\theta_{\student})) := f_{\theta_{\student}}(\vX_{\maskcrop}), \\
    &(\vz_{\globalview}^{\cls}(\theta_{\teacher}), \vZ_{\globalview}^{\patch}(\theta_{\teacher})) := f_{\theta_{\teacher}}(\vX_{\globalview}), \\
    &\vp_{\maskcrop}^{\cls}(\theta_{\student}, \eta_{\student}^{\dino}) := \softmax(h_{\eta_{\student}^{\dino}}(\vz_{\maskcrop}^{\cls}(\theta_{\student}))/\tau), \\
    &\vp_{\maskcrop}^{i}(\theta_{\student}, \eta_{\student}^{\ibot}) := \softmax(h_{\eta_{\student}^{\ibot}}(\vz_{\maskcrop}^{i}(\theta_{\student}))/\tau), \\
    &\scalebox{0.9}{\(\vp_{\globalview}^{\cls}(\theta_{\teacher}, \eta_{\teacher}^{\dino}) := \softmax(\textsc{SKC}[h_{\eta_{\teacher}^{\dino}}(\vz_{\globalview}^{\cls}(\theta_{\teacher}))]/\tau)\)}, \\
    &\scalebox{0.95}{\(\vp_{\globalview}^{i}(\theta_{\teacher}, \eta_{\teacher}^{\ibot}) := \softmax(\textsc{SKC}[h_{\eta_{\teacher}^{\ibot}}(\vz_{\globalview}^{i}(\theta_{\teacher}))]/\tau)\)}.
\end{align}
We then compute the loss using all probability vectors:
\begin{align}
    &\cL_{\dino{}\mathrm{v2}}(\theta_{\student}, \theta_{\teacher}, \eta_{\student}^{\dino}, \eta_{\student}^{\ibot}, \eta_{\teacher}^{\dino}, \eta_{\teacher}^{\ibot}) := \\
    &\frac{1}{2}\scalebox{0.9}{\(\displaystyle\Ex\mathopen{}\Bigg[d_{\CE}(\vp_{\maskcrop}^{\cls}(\theta_{\student}, \eta_{\student}^{\dino}), \vp_{\globalview}^{\cls}(\theta_{\teacher}, \eta_{\teacher}^{\dino}))\)} \nonumber \\
    &\scalebox{0.9}{\(\displaystyle\qquad + \frac{1}{N}\sum_{i = 1}^{N}d_{\CE}(\vp_{\maskcrop}^{i}(\theta_{\student}, \eta_{\student}^{\ibot}), \vp_{\globalview}^{i}(\theta_{\teacher}, \eta_{\teacher}^{\ibot}))\bm{1}_{\maskcrop{}i}\Bigg]\)} \nonumber \\ 
    &\scalebox{0.9}{\(\displaystyle -\gamma \operatorname{Entropy}(\vz_{\maskcrop}^{\cls}(\theta_{\student}))\)}, \nonumber
\end{align}
where \(\bm{1}_{\maskcrop{}i}\) is \(1\) if patch \(i\) is masked by \(v_{\maskcrop}\) and \(0\) otherwise, and the \(\operatorname{Entropy}\) functional is the differential entropy; it plays a similar role as the coding rate \(R_{\eps}\) in \simdino{} (and shortly \simdino{}v2) in ensuring non-collapse. It is estimated by \citet{oquab2023dinov2} using the KoLeo estimator \citep{delattre2017kozachenko}) which explicitly uses negative samples. However, the KoLeo estimator is a non-parametric estimator of the expectation of a function of a high-dimensional probability density \citep{beirlant1997nonparametric}, and so it has relatively poor sample efficiency (i.e., the required batch size to converge in practice is large).

\textbf{We now greatly simplify the above pipeline} using the same ideas as introduced in \simdino{}. Namely, we dispense with the \dino{}/\ibot{} heads, the Sinkhorn-Knopp centering, and the softmaxes, and compute the Euclidean distance-based loss directly on normalized features. We obtain the loss 
\vspace{-0.5em}
\begin{align}\label{eq:simdinov2_loss}
    &\cL_{\mathrm{\simdino{}v2}}(\theta_{\student}, \theta_{\teacher}) 
    := \frac{1}{2}\scalebox{0.86}{\(\displaystyle\Ex\mathopen{}\Bigg[d_{\ell^{2}}(\vz_{\maskcrop}^{\cls}(\theta_{\student}), \vz_{\globalview}^{\cls}(\theta_{\teacher}))\)} \\
    &\scalebox{0.86}{\(\displaystyle+ \frac{1}{N}\sum_{i = 1}^{N}d_{\ell^{2}}(\vz_{\maskcrop}^{i}(\theta_{\student}), \vz_{\globalview}^{i}(\theta_{\teacher})) \bm{1}_{\maskcrop{}i}\Bigg] -\gamma R_{\eps}\big(\Cov[\vz_{\maskcrop}^{\cls}(\theta_{\student})]\big)\)}. \nonumber
\end{align}
The same caveats as in \simdino{} apply with respect to how the expectations and covariances are estimated, and the optimization and evaluation procedures carry over. We provide pseudocode for the training pipeline in \Cref{alg:simdinov2_training_pipeline} in \Cref{app:pipeline_pseudocode}. In the sequel, we will show that these greatly simplified designs actually help the model performance.

\vspace{-0.75em}

\paragraph{Optimal value for \(\gamma\).}
In both the \simdino{} loss \eqref{eq:simdino_loss} and the \simdino{}v2 loss \eqref{eq:simdinov2_loss}, in order to aid learning while making sure neither the distance term nor the regularizer term dominates, we choose \(\gamma\) up to an absolute constant factor so that it balances the asymptotic order of the gradient (Frobenius) norms of both terms. By the Cauchy-Schwarz inequality, it suffices to equalize the norms of the gradients of each term w.r.t.~the features \(\vZ_{\crop}\). Since the features are normalized on the sphere, it holds that the gradient norm of the distance term is \(\order(1)\). For the second term, assuming that we use \(n\) samples to estimate the covariance, \Cref{thm:optimizer_R_eps} (\Cref{sec:theory}) says that the gradient norm of the second term is \(\order(\sqrt{d\min\{d, n\}/n}/\eps)\). To make these equivalent, we take \(\gamma = \Theta(\eps\sqrt{n/(d\min\{d, n\})})\). The same rate holds for \simdino{}v2.  We recognize that this choice of \(\gamma\) is ultimately a heuristic, and the constant factor needs to be tuned, but it helps to scale \simdino{} and \simdino{}v2 in practice.

\vspace{-0.5em}
\section{Experimental Verification} \label{sec:experiments}
In this section, we empirically investigate and evaluate our proposed \simdino{} and \simdino{}v2 models and compare them to the original \dino{} and \dino{}v2 model families. In particular, we examine their differences in training dynamics and learned representation both quantitatively and qualitatively. 
Overall, our experiments show that our proposed \simdino{} model families can achieve better performance and learn representations of higher quality than the original \dino{} families while being significantly simpler and more robust to variations in hyperparameters and architecture.

\subsection{Experimental Setup}

\paragraph{Model architecture.}
Since our method is directly built upon \dino{} and \dino{}v2, we adopt settings as close as possible to the original method for fair comparison. Specifically, for all inputs we set patch size to be 16; we use the small, base, and large models of the ViT \citep{dosovitskiy2020image} architecture as the backbone, which is connected to a projector composed of three MLP layers with a hidden size of 2048 and an output dimension of 256. The output features after the projector are $\ell^2$ normalized. Specifically for original (i.e., unsimplified) \dino{} models, these normalized features are then fed to a weight-normalized linear layer that outputs a high-dimensional (e.g., 65536) vector, before computing the softmax and then the cross-entropy loss.

\paragraph{Datasets and optimization.} For pretraining, we use the ImageNet-1K dataset across all methods. For fair comparison, we closely follow the original works \citep{caron2021emerging, oquab2023dinov2}. We choose AdamW \citep{loshchilov2017decoupled} as the optimizer and adopt the same optimization strategies (e.g., learning rates, warm-up schedules). For multicrop augmentation, we use 10 local views of resolution $96 \times 96$ and 2 global views of resolution $224 \times 224$ for all experiments. We provide more details on hyperparameter choices in~\Cref{app:impl_details}. We also consider several downstream tasks. Specifically, we evaluate our pretrained models on 1) unsupervised object detection and segmentation on COCO {val2017} \citep{lin2014mscoco}, 2) semantic segmentation on ADE20K \citep{zhou2017scene}, and 3) video object segmentation on DAVIS-2017 \citep{Pont-Tuset_arXiv_2017}.

\subsection{Experimental Results}

\paragraph{ImageNet Classification.} We report the classification accuracies on ImageNet-1k in~\Cref{tab:imagenet_performance}. Following \citep{caron2021emerging}, we evaluate both $k$-NN and linear accuracy on the ViT backbones pretrained by the \dino{} model families and our simplified variants. We observe that under both \dino{} and \dino{}v2 paradigms, our simplified methods are able to outperform the original pipelines. Furthermore, we observe that applying identical hyperparameter settings from ViT-B to ViT-L results in instability and divergence in \dino{}, while the same setup yields a steady improvement for \simdino{}. To better understand the optimization dynamics of \simdino{}, we visualize the evolution of accuracy during training in~\Cref{fig:dinov1_knn}. It can be observed that performance of \simdino{} steadily improves as training progresses, while optimization of \dino{} noticeably slows down, with even a slight performance drop near the end of training. Together, these results demonstrate our simplified pipelines' stability and ease of optimization compared to the originals.

\begin{figure}[ht]
    \centering
    \includegraphics[width=\columnwidth]{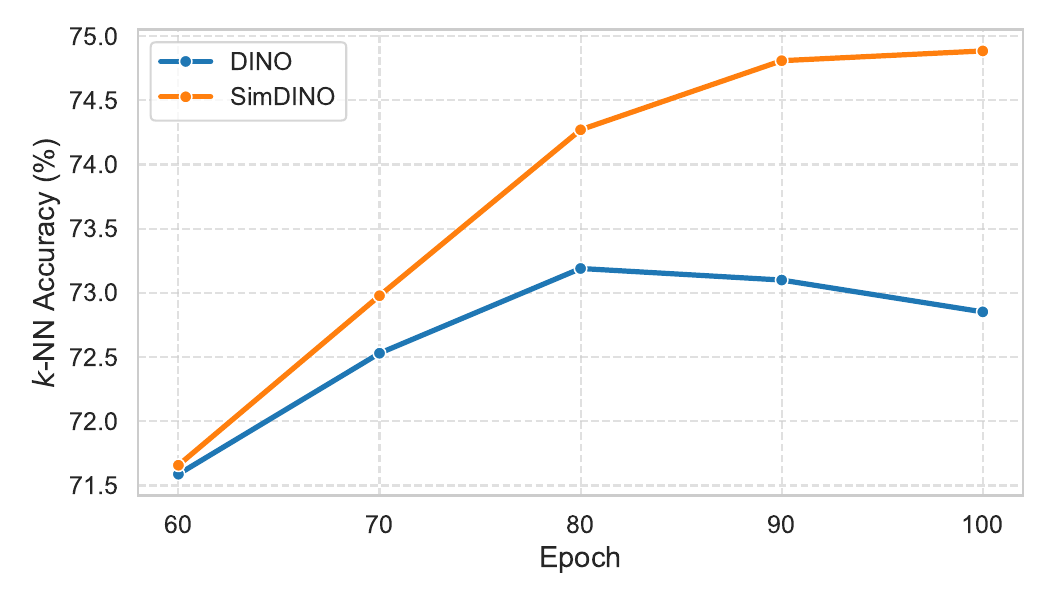}

    \vspace{-1.5em}
    
    \caption{\small\textbf{Evolution of $k$-NN accuracy} of ViT-B trained for 100 epochs using \theirs{} and \ours{} paradigms on ImageNet-1K. We omit earlier epochs of similar metrics for better visual clarity. 
    }
    \label{fig:dinov1_knn}
    \vspace{-1em}
\end{figure}

\begin{table}[ht]
    \caption{\small\textbf{Performance comparison on ImageNet-1K.} \simdino{} and \simdino{}v2 consistently outperform the original \dino{} and \dino{}v2 model families. They are also more stable, while training of \dino{} on ViT-L diverged (row 3).}
    \centering
    
    \vspace{0.5em}
    
    \begin{tabular}{@{}lcccc@{}} %
        \toprule
        Method & Model & Epochs & $k$-NN & Linear
        \\
        \midrule 
        DINO & ViT-B & 100 & 72.9 & 76.3 \\
        SimDINO & ViT-B & 100 & \bf 74.9 & \bf 77.3 \\
        DINO & ViT-L & 100 & -- & -- \\
        SimDINO & ViT-L & 100 & \bf 75.6 & \bf 77.4 \\
        \midrule
        DINOv2 & ViT-B & 100 & 76.0 & 77.2 \\
        SimDINOv2 & ViT-B & 100 & \bf 78.1 & \bf 79.7 \\
        DINOv2 & ViT-L & 100 & 80.8 & 82.0 \\
        SimDINOv2 & ViT-L & 100 & \bf 81.1 & \bf 82.4 \\
        \midrule
        \color{gray} SwAV & \color{gray} ViT-S & \color{gray} 800 & \color{gray} 66.3 & \color{gray} 73.5 \\
        \color{gray} MoCov3 & \color{gray} ViT-B & \color{gray} 300  & \color{gray} -- & \color{gray} 76.7 \\
        \bottomrule
    \end{tabular}
    \label{tab:imagenet_performance}
    \vspace{-1em}
\end{table}

\vspace{-0.5em}

\paragraph{Object Detection and Segmentation.}
To better understand the learned representation, we evaluate the pretrained models on segmentation and object detection tasks. Specifically, we adopt MaskCut \citep{wang2023cut}, an effective unsupervised approach of extracting features from a frozen vision backbone for object detection and instance segmentation. In~\Cref{fig:maskcut_visualization}, we present qualitative segmentation results by applying MaskCut on models trained with both \dino{} and \simdino{}. Both methods are observed to produce meaningful segmentation results, confirming the emerging properties similar to the original \dino{} when using our simplified algorithm. More qualitative results are available in~\Cref{app:vis}. To quantitatively evaluate these representation, we perform MaskCut on the COCO {val2017} dataset and report our results in~\Cref{tab: maskcut_main_result}. These results show \simdino{} achieves much stronger performance on segmentation and detection tasks than \dino{} when trained on the same network (row 2 vs 3), and overall even outperforms \dino{} trained on a smaller patch size\footnote{When trained using \dino{}, ViT models with smaller patch sizes tend to outperform those with larger ones on various tasks including segmentation \citep{wang2023cut,caron2021emerging}. } (row 2 vs 4). 

\vspace{-0.5em}

\paragraph{Semantic Segmentation on ADE20K.} We evaluate our proposed methods on the ADE20K semantic segmentation task and report the results in~\Cref{ade20kanddavis_main_result} (column 3 \& 4). Specifically, we follow the linear evaluation protocol of \citep{zhou2021ibot}, where we fix the pretrained backbone and only finetune a linear layer on top of it. From the results, we observe that our proposed \ours{} consistently outperforms the original algorithms. In particular, on ViT-B, \ours{}v2 is able to improve \dino{}v2 by $4.4$@mIoU. These results suggest that our simplified methods lead to representations favorable to dense prediction tasks.

\vspace{-0.5em}

\paragraph{DAVIS Video Object Segmentation.} In~\Cref{ade20kanddavis_main_result}, we also provide evaluation results on DAVIS-2017 video instance segmentation benchmark. We follow the same evaluation protocol as in \citep{caron2021emerging} and segment scenes between consecutive video frames with nearest neighbor. We observe that our proposed \simdino{}(v2) outperforms the original methods on this task. One interesting observation is that despite achieving much better $k-$NN accuracy, \dino{}v2 generally underperforms the original \dino{} in this task (and similarly for the simplified variants). A similar phenomenon is noted in \citep{zhou2021ibot}, where this discrepancy is found to be caused by the sensitivity of the evaluation protocol itself (e.g., image resolution). In our evaluation, we do not tune these individual factors and simply adopt the same setting across all models we consider.

\vspace{-0.5em}

\paragraph{More on Stability and Robustness.}
Apart from the observed divergence on ViT-L in~\Cref{tab:imagenet_performance}, we note that \dino{} is sensitive to its pipeline-specific hyperparameters, as evidenced in~\Cref{tab:dino_hparam_sensitivity} (in \Cref{app:more_expts}). To further verify the stability of \simdino{}, we experiment with training both algorithms on a different dataset than ImageNet-1k. Specificlly, we train them on COCO train2017 (roughly $1/10$-th the size of ImageNet-1k), and report the results in~\Cref{fig:dinov1_knn_coco}. Under this setting, \simdino{} vastly outperforms \dino{}. We provide additional ablations on other factors (e.g. batch sizes) in~\Cref{app:more_expts}. Together, these results demonstrate the superior stability and robustness of \simdino{}.

\begin{figure*}[t!]
    \centering
    \includegraphics[width=0.90\linewidth]{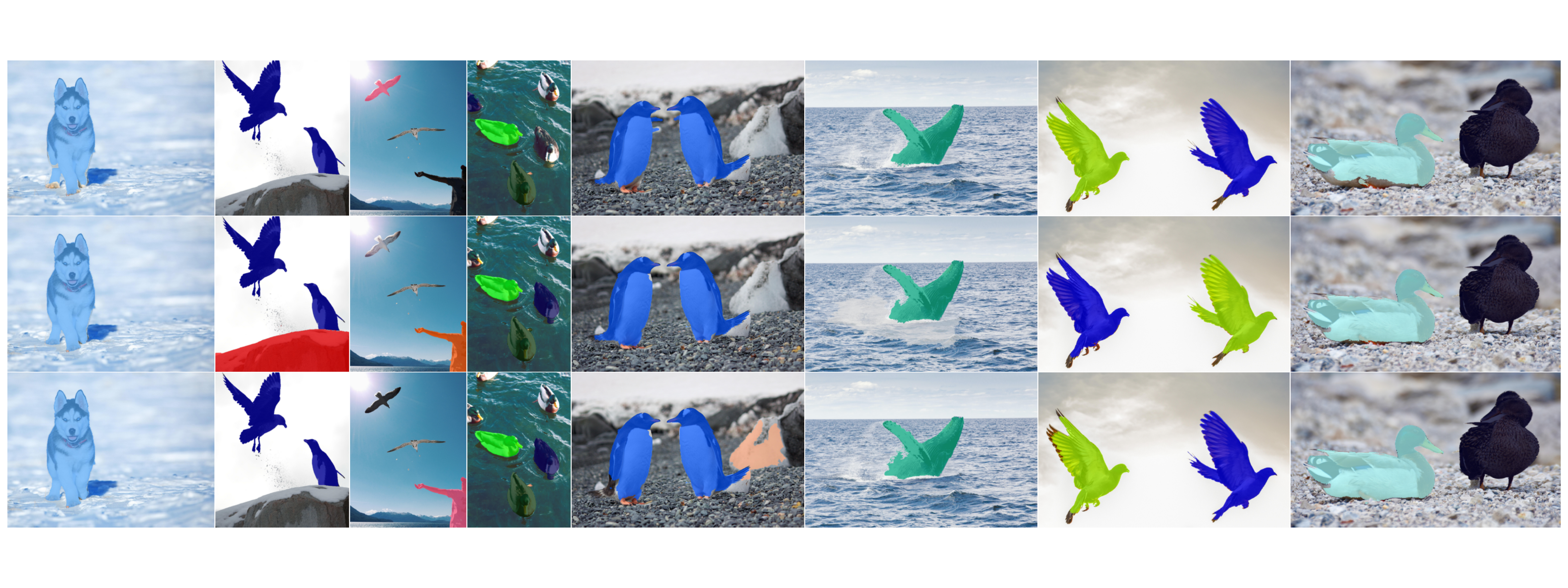}

    \vspace{-2.5em}
    
    \caption{\small \textbf{Visualization of MaskCut segmentation results} from DINO ViT-B/16 (row 1), SimDINO ViT-B/16 (row 2) and SimDINO ViT-L/16 (row 3) on selected images.
    }
    \label{fig:maskcut_visualization}
    \vspace{-1.5em}
\end{figure*}

\begin{figure}[t!]
    \centering
    \includegraphics[width=\columnwidth]{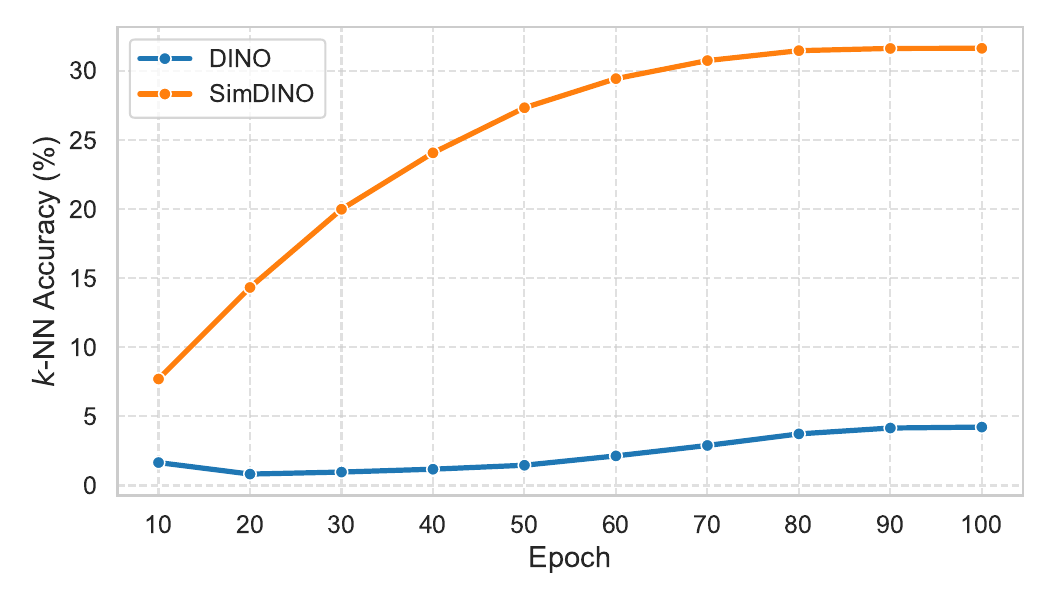}
    
    \vspace{-1.5em}
    
    \caption{\small\textbf{$k$-NN accuracy on ImageNet-1K} of ViT-B trained on \textbf{COCO train2017} using \theirs{} and \ours{} paradigms.  
    }
    \label{fig:dinov1_knn_coco}
    \vspace{-1em}
\end{figure}

\begin{table}[t!]
    \centering
    \small
    \setlength{\tabcolsep}{4pt}
    \caption{\textbf{Unsupervised object detection and segmentation via MaskCut  evaluated on COCO {val2017}} under COCO's official evaluation protocol. 
    \ours{} conclusively performs better than the \dino{} at detection and segmentation metrics, comparable with \dino{} with smaller path size (16 vs 8). 
    }
    \vspace{0.25em}
    \begin{tabular}{@{}llcccccccc@{}}
        \toprule
         &  & \multicolumn{3}{c}{Detection $\uparrow$} &  \multicolumn{3}{c}{Segmentation $\uparrow$} \\ 
        Method & Model & AP$_{50}$  & AP$_{75}$ & AP & AP$_{50}$ & AP$_{75}$ & AP  \\ 
        \midrule
        \ours{} &ViT-L/16 &\bf 5.4 &1.9 &2.4 &4.5 &1.4 &1.9 \\
        \ours{} &ViT-B/16 &5.2 & \bf 2.0 & \bf 2.5 & \bf4.7 & \bf 1.5 & \bf 2.0 \\
        \dino{} &ViT-B/16 &3.9 &1.5 &1.8 &3.1 &1.0 &1.4 \\
        \midrule
        \color{gray} \dino{} & \color{gray} ViT-B/8 & \color{gray}5.1 & \color{gray}2.3 & \color{gray}2.5 & \color{gray}4.1 & \color{gray}1.3 & \color{gray}1.8 \\
        \bottomrule
    \end{tabular}
    \label{tab: maskcut_main_result}
    \vspace{-1.5em}
\end{table}

\begin{table}[t!]
    \centering
    \small
    \setlength{\tabcolsep}{4pt}
    \caption{\textbf{Semantic segmentation on ADE20K and video object segmentation on DAVIS-2017.} For semantic segmentation, we train a linear layer on the frozen pretrained backbone. On DAVIS, we segment scenes between video frames using nearest neighbor search. On both tasks, \ours{}(v2) consistently outperforms their original counterparts.
    }
    \vspace{0.25em}
    \begin{tabular}{@{}lcccccccc@{}}
        \toprule
        &  & \multicolumn{2}{c}{Lin. Seg. $\uparrow$} &  \multicolumn{3}{c}{Vid. Seg. $\uparrow$} \\
        Method & Model & mIoU  & mAcc & $(\mathcal{J} \& \mathcal{F})_m$  & $\mathcal{J}_m$ & $\mathcal{F}_m$\\ 
        \midrule
        \dino{} &ViT-B/16 &33.1 & 41.9 & 63.0 & 61.5 & 64.4 \\
        \ours{} &ViT-B/16 & \bf 33.7 & \bf 42.8 & \bf 63.0 & \bf 61.6 & \bf 64.4 \\
        \dino{}v2 &ViT-B/16 &32.5 &41.4 & 53.2 & 52.7 & 53.7    \\
        \ours{}v2 &ViT-B/16 & \bf 36.9 & \bf 46.5 & \bf 60.9 & \bf 60.4& \bf 61.4  \\
        \dino{}v2 &ViT-L/16 & 41.0 & 50.8 &62.0 & 61.7& 62.3  \\
        \ours{}v2 &ViT-L/16 & \bf 41.8 & \bf 52.2 & \bf 62.6 & \bf 61.9 & \bf 63.3   \\
        \bottomrule
    \end{tabular}
    \label{ade20kanddavis_main_result}
    \vspace{-1.5em}
\end{table}

\section{Related Work} \label{sec:related_work}

In this section, we identify several previous works which the \simdino{} and \simdino{}v2 methodologies are similar to or build on. We have already discussed similarities to \dino{} and \dino{}v2 in depth so we omit this comparison.

\paragraph{Siamese contrastive SSL.} Siamese contrastive learning, archetyped by SimCLR \citep{chen2020simple} and SimSiam \citep{chen2021exploring} among others, uses the same network to encode different augmentations (i.e., views) of the same input, and pushes the features of these augmentations together, similar to \simdino{}. SimCLR uses explicit negative samples in the loss, while SimSiam manipulates the loss gradient structure using stop-gradients to avoid collapse. Both methods' losses measure alignment or difference via the squared Euclidean distance (equivalently the dot product) of the features. In contrast, \simdino{} uses two separate networks --- the teacher and student --- that update via self-distillation. Furthermore, \simdino{} uses the inner product of features in the loss, but it also uses a coding rate regularizer instead of implicitly contrasting negative samples or using the more bespoke contrastive loss in SimCLR.

\paragraph{Explicit covariance regularization in SSL.} There have also been works that use explicit penalization of the first- and second-order statistics of the features, such as VICReg \citep{bardes2021vicreg}. VICReg uses completely separate networks to encode two augmentations of the same input batch, and then explicitly penalizes the alignment of those features (via Euclidean distance) as well as the features' variance and covariance within the batch, aiming to whiten the features as much as possible. In spirit, this is similar to \simdino{}, which also penalizes the alignment and the features' covariance, albeit using a different covariance regularizer and not penalizing the features' variance. Also, \simdino{} uses self-distillation to train the teacher network, while VICReg uses two separate networks.

\paragraph{Self-distillation in SSL.} Several works such as MoCo \citep{he2020momentum} and BYOL \citep{grill2020bootstrap} train two networks, a teacher and a student, via self-distillation by setting the teacher weights to be an exponential moving average of the student weights. While MoCo uses explicit negative samples from previous batches in its InfoNCE loss computed on a given batch, BYOL does not use negative samples but instead manipulates the gradient structure (akin to SimSiam) in order to prevent collapse, and it uses an extra (``prediction'') module appended to the student network, making the teacher and student asymmetric. \simdino{} uses self-distillation with the same architecture for teacher and student, explicitly uses the simple Euclidean distance in the loss, and explicitly uses the coding rate to prevent collapse.

\paragraph{Patch feature prediction in SSL.} While most contrastive SSL methods pick a single feature vector (say, of the \(\cls\) token) as the representation, recent contrastive learning approaches such as \dino{}v2, I-JEPA \citep{assran2023self}, and C-JEPA \citep{mo2024connecting} compute losses on the features corresponding to each patch. In I-JEPA, there is one local and one global view, whose crops are nested, and the (Euclidean distance) loss is only computed on the patch features. C-JEPA adds a VICReg-esque variance and covariance penalty to the objective of I-JEPA. In contrast, in \simdino{}v2, there are multiple local and global views, the loss incorporates both patch-based and aggregate features, and collapse is prevented by using a coding rate term.

\paragraph{Coding rate, and related regularizers.} Several works have used coding rate-related terms in the objective \citep{ma2007segmentation,yu2020learning,dai2022ctrl,tong2022unsupervised} as well as a way to evaluate quality of representations \citep{yu2023white,pai2023masked,wu2024token,yang2024scaling}. The coding rate has thus been shown to provide a powerful measure for non-collapse or expansion of the features from a given batch. Other regularizers to accomplish this include the VICReg-type regularizers and the MMCR regularizer \citep{yerxa2023learning,schaeffer2024towards}.

\section{Conclusion} \label{sec:conclusion}

In this work, we identify that the reasons for many empirically motivated design choices in the original \dino{} and \dino{}v2 are to avoid collapse of the learned representation. We show that these complicated design choices can be significantly reduced or simplified by adding a coding-rate-related regularization term. The resulting simplified models, called \simdino{} and \simdino{}v2, are even better in terms of performance for downstream tasks, and their pretraining pipelines are much more robust to different settings and hyperparameters, offering a Pareto improvement against the \dino{} and \dino{}v2 model families. Our work demonstrates the value of simplifying deep learning pipelines as well as making tradeoffs as explicit as possible when designing high-performance vision SSL models.

In light of these overarching contributions, there are several possible opportunities for future work. On the theoretical side, our simplified framework provides an entry point for studying the geometric properties of the global optima of self-supervised learning losses. Further study in \Cref{sub:dino_without_teacherstudent} shows that in the framework of the paper, it is possible to set up a self-supervised objective that does not require self-distillation to optimize, making a theoretical analysis much easier, while the resulting model is still quite powerful and practically usable. On the empirical side, one can apply the paradigm of making implicit design choices more explicitly present in the loss to more self-supervised learning frameworks, making existing pipelines more stable and the resulting models of better performance.

\FloatBarrier
\bibliography{reference}

\begin{thebibliography}{46}
\providecommand{\natexlab}[1]{#1}
\providecommand{\url}[1]{\texttt{#1}}
\expandafter\ifx\csname urlstyle\endcsname\relax
  \providecommand{\doi}[1]{doi: #1}\else
  \providecommand{\doi}{doi: \begingroup \urlstyle{rm}\Url}\fi

\bibitem[Assran et~al.(2023)Assran, Duval, Misra, Bojanowski, Vincent, Rabbat, LeCun, and Ballas]{assran2023self}
Assran, M., Duval, Q., Misra, I., Bojanowski, P., Vincent, P., Rabbat, M., LeCun, Y., and Ballas, N.
\newblock Self-supervised learning from images with a joint-embedding predictive architecture.
\newblock In \emph{Proceedings of the IEEE/CVF Conference on Computer Vision and Pattern Recognition}, pp.\  15619--15629, 2023.

\bibitem[Baharoon et~al.(2023)Baharoon, Qureshi, Ouyang, Xu, Phol, Aljouie, and Peng]{baharoon2023towards}
Baharoon, M., Qureshi, W., Ouyang, J., Xu, Y., Phol, K., Aljouie, A., and Peng, W.
\newblock Towards general purpose vision foundation models for medical image analysis: An experimental study of dinov2 on radiology benchmarks.
\newblock \emph{arXiv preprint arXiv:2312.02366}, 2023.

\bibitem[Bardes et~al.(2021)Bardes, Ponce, and LeCun]{bardes2021vicreg}
Bardes, A., Ponce, J., and LeCun, Y.
\newblock Vicreg: Variance-invariance-covariance regularization for self-supervised learning.
\newblock \emph{arXiv preprint arXiv:2105.04906}, 2021.

\bibitem[Beirlant et~al.(1997)Beirlant, Dudewicz, Gy{\"o}rfi, Van~der Meulen, et~al.]{beirlant1997nonparametric}
Beirlant, J., Dudewicz, E.~J., Gy{\"o}rfi, L., Van~der Meulen, E.~C., et~al.
\newblock Nonparametric entropy estimation: An overview.
\newblock \emph{International Journal of Mathematical and Statistical Sciences}, 6\penalty0 (1):\penalty0 17--39, 1997.

\bibitem[Brown et~al.(2020)Brown, Mann, Ryder, Subbiah, Kaplan, Dhariwal, Neelakantan, Shyam, Sastry, Askell, et~al.]{brown2020language}
Brown, T., Mann, B., Ryder, N., Subbiah, M., Kaplan, J.~D., Dhariwal, P., Neelakantan, A., Shyam, P., Sastry, G., Askell, A., et~al.
\newblock Language models are few-shot learners.
\newblock \emph{Advances in neural information processing systems}, 33:\penalty0 1877--1901, 2020.

\bibitem[Caron et~al.(2020)Caron, Misra, Mairal, Goyal, Bojanowski, and Joulin]{caron2020unsupervised}
Caron, M., Misra, I., Mairal, J., Goyal, P., Bojanowski, P., and Joulin, A.
\newblock Unsupervised learning of visual features by contrasting cluster assignments.
\newblock \emph{Advances in neural information processing systems}, 33:\penalty0 9912--9924, 2020.

\bibitem[Caron et~al.(2021)Caron, Touvron, Misra, J{\'e}gou, Mairal, Bojanowski, and Joulin]{caron2021emerging}
Caron, M., Touvron, H., Misra, I., J{\'e}gou, H., Mairal, J., Bojanowski, P., and Joulin, A.
\newblock Emerging properties in self-supervised vision transformers.
\newblock In \emph{Proceedings of the IEEE/CVF international conference on computer vision}, pp.\  9650--9660, 2021.

\bibitem[Chen et~al.(2020)Chen, Kornblith, Norouzi, and Hinton]{chen2020simple}
Chen, T., Kornblith, S., Norouzi, M., and Hinton, G.
\newblock A simple framework for contrastive learning of visual representations.
\newblock In \emph{International conference on machine learning}, pp.\  1597--1607. PMLR, 2020.

\bibitem[Chen \& He(2021)Chen and He]{chen2021exploring}
Chen, X. and He, K.
\newblock Exploring simple siamese representation learning.
\newblock In \emph{Proceedings of the IEEE/CVF conference on computer vision and pattern recognition}, pp.\  15750--15758, 2021.

\bibitem[Cuturi(2013)]{cuturi2013sinkhorn}
Cuturi, M.
\newblock Sinkhorn distances: Lightspeed computation of optimal transport.
\newblock \emph{Advances in neural information processing systems}, 26, 2013.

\bibitem[Dai et~al.(2022)Dai, Tong, Li, Wu, Psenka, Chan, Zhai, Yu, Yuan, Shum, et~al.]{dai2022ctrl}
Dai, X., Tong, S., Li, M., Wu, Z., Psenka, M., Chan, K. H.~R., Zhai, P., Yu, Y., Yuan, X., Shum, H.-Y., et~al.
\newblock Ctrl: Closed-loop transcription to an ldr via minimaxing rate reduction.
\newblock \emph{Entropy}, 24\penalty0 (4):\penalty0 456, 2022.

\bibitem[Delattre \& Fournier(2017)Delattre and Fournier]{delattre2017kozachenko}
Delattre, S. and Fournier, N.
\newblock On the kozachenko--leonenko entropy estimator.
\newblock \emph{Journal of Statistical Planning and Inference}, 185:\penalty0 69--93, 2017.

\bibitem[Devlin(2018)]{devlin2018bert}
Devlin, J.
\newblock Bert: Pre-training of deep bidirectional transformers for language understanding.
\newblock \emph{arXiv preprint arXiv:1810.04805}, 2018.

\bibitem[Dosovitskiy(2020)]{dosovitskiy2020image}
Dosovitskiy, A.
\newblock An image is worth 16x16 words: Transformers for image recognition at scale.
\newblock \emph{arXiv preprint arXiv:2010.11929}, 2020.

\bibitem[Feichtenhofer et~al.(2022)Feichtenhofer, Li, He, et~al.]{feichtenhofer2022masked}
Feichtenhofer, C., Li, Y., He, K., et~al.
\newblock Masked autoencoders as spatiotemporal learners.
\newblock \emph{Advances in neural information processing systems}, 35:\penalty0 35946--35958, 2022.

\bibitem[Grill et~al.(2020)Grill, Strub, Altch{\'e}, Tallec, Richemond, Buchatskaya, Doersch, Avila~Pires, Guo, Gheshlaghi~Azar, et~al.]{grill2020bootstrap}
Grill, J.-B., Strub, F., Altch{\'e}, F., Tallec, C., Richemond, P., Buchatskaya, E., Doersch, C., Avila~Pires, B., Guo, Z., Gheshlaghi~Azar, M., et~al.
\newblock Bootstrap your own latent-a new approach to self-supervised learning.
\newblock \emph{Advances in neural information processing systems}, 33:\penalty0 21271--21284, 2020.

\bibitem[Hadsell et~al.(2006)Hadsell, Chopra, and LeCun]{hadsell2006dimensionality}
Hadsell, R., Chopra, S., and LeCun, Y.
\newblock Dimensionality reduction by learning an invariant mapping.
\newblock In \emph{2006 IEEE computer society conference on computer vision and pattern recognition (CVPR'06)}, volume~2, pp.\  1735--1742. IEEE, 2006.

\bibitem[He et~al.(2020)He, Fan, Wu, Xie, and Girshick]{he2020momentum}
He, K., Fan, H., Wu, Y., Xie, S., and Girshick, R.
\newblock Momentum contrast for unsupervised visual representation learning.
\newblock In \emph{Proceedings of the IEEE/CVF conference on computer vision and pattern recognition}, pp.\  9729--9738, 2020.

\bibitem[He et~al.(2022)He, Chen, Xie, Li, Doll{\'a}r, and Girshick]{he2022masked}
He, K., Chen, X., Xie, S., Li, Y., Doll{\'a}r, P., and Girshick, R.
\newblock Masked autoencoders are scalable vision learners.
\newblock In \emph{Proceedings of the IEEE/CVF conference on computer vision and pattern recognition}, pp.\  16000--16009, 2022.

\bibitem[Kahn \& Marshall(1953)Kahn and Marshall]{kahn1953methods}
Kahn, H. and Marshall, A.~W.
\newblock Methods of reducing sample size in monte carlo computations.
\newblock \emph{Journal of the Operations Research Society of America}, 1\penalty0 (5):\penalty0 263--278, 1953.

\bibitem[Li et~al.(2022)Li, Chen, LeCun, and Sommer]{li2022neural}
Li, Z., Chen, Y., LeCun, Y., and Sommer, F.~T.
\newblock Neural manifold clustering and embedding.
\newblock \emph{arXiv preprint arXiv:2201.10000}, 2022.

\bibitem[Lin et~al.(2014)Lin, Maire, Belongie, Hays, Perona, Ramanan, Doll{\'a}r, and Zitnick]{lin2014mscoco}
Lin, T.-Y., Maire, M., Belongie, S., Hays, J., Perona, P., Ramanan, D., Doll{\'a}r, P., and Zitnick, C.~L.
\newblock Microsoft coco: Common objects in context.
\newblock In \emph{Computer Vision--ECCV 2014: 13th European Conference, Zurich, Switzerland, September 6-12, 2014, Proceedings, Part V 13}, pp.\  740--755. Springer, 2014.

\bibitem[Loshchilov(2017)]{loshchilov2017decoupled}
Loshchilov, I.
\newblock Decoupled weight decay regularization.
\newblock \emph{arXiv preprint arXiv:1711.05101}, 2017.

\bibitem[Ma et~al.(2007)Ma, Derksen, Hong, and Wright]{ma2007segmentation}
Ma, Y., Derksen, H., Hong, W., and Wright, J.
\newblock Segmentation of multivariate mixed data via lossy data coding and compression.
\newblock \emph{IEEE transactions on pattern analysis and machine intelligence}, 29\penalty0 (9):\penalty0 1546--1562, 2007.

\bibitem[Mo \& Tong(2024)Mo and Tong]{mo2024connecting}
Mo, S. and Tong, S.
\newblock Connecting joint-embedding predictive architecture with contrastive self-supervised learning.
\newblock \emph{arXiv preprint arXiv:2410.19560}, 2024.

\bibitem[Oord et~al.(2018)Oord, Li, and Vinyals]{oord2018representation}
Oord, A. v.~d., Li, Y., and Vinyals, O.
\newblock Representation learning with contrastive predictive coding.
\newblock \emph{arXiv preprint arXiv:1807.03748}, 2018.

\bibitem[Oquab et~al.(2023)Oquab, Darcet, Moutakanni, Vo, Szafraniec, Khalidov, Fernandez, Haziza, Massa, El-Nouby, et~al.]{oquab2023dinov2}
Oquab, M., Darcet, T., Moutakanni, T., Vo, H., Szafraniec, M., Khalidov, V., Fernandez, P., Haziza, D., Massa, F., El-Nouby, A., et~al.
\newblock Dinov2: Learning robust visual features without supervision.
\newblock \emph{arXiv preprint arXiv:2304.07193}, 2023.

\bibitem[Pai et~al.(2023)Pai, Wu, Buchanan, Yu, and Ma]{pai2023masked}
Pai, D., Wu, Z.~W., Buchanan, S., Yu, Y., and Ma, Y.
\newblock Masked completion via structured diffusion with white-box transformers.
\newblock International Conference on Learning Representations, 2023.

\bibitem[Pont-Tuset et~al.(2017)Pont-Tuset, Perazzi, Caelles, Arbel\'aez, Sorkine-Hornung, and {Van Gool}]{Pont-Tuset_arXiv_2017}
Pont-Tuset, J., Perazzi, F., Caelles, S., Arbel\'aez, P., Sorkine-Hornung, A., and {Van Gool}, L.
\newblock The 2017 davis challenge on video object segmentation.
\newblock \emph{arXiv:1704.00675}, 2017.

\bibitem[Radford et~al.(2018)Radford, Narasimhan, Salimans, and Sutskever]{radford2018improving}
Radford, A., Narasimhan, K., Salimans, T., and Sutskever, I.
\newblock Improving language understanding by generative pre-training.
\newblock 2018.

\bibitem[Radford et~al.(2019)Radford, Wu, Child, Luan, Amodei, Sutskever, et~al.]{radford2019language}
Radford, A., Wu, J., Child, R., Luan, D., Amodei, D., Sutskever, I., et~al.
\newblock Language models are unsupervised multitask learners.
\newblock \emph{OpenAI blog}, 1\penalty0 (8):\penalty0 9, 2019.

\bibitem[Radford et~al.(2021)Radford, Kim, Hallacy, Ramesh, Goh, Agarwal, Sastry, Askell, Mishkin, Clark, et~al.]{radford2021learning}
Radford, A., Kim, J.~W., Hallacy, C., Ramesh, A., Goh, G., Agarwal, S., Sastry, G., Askell, A., Mishkin, P., Clark, J., et~al.
\newblock Learning transferable visual models from natural language supervision.
\newblock In \emph{International conference on machine learning}, pp.\  8748--8763. PMLR, 2021.

\bibitem[Salimans \& Kingma(2016)Salimans and Kingma]{salimans2016weight}
Salimans, T. and Kingma, D.~P.
\newblock Weight normalization: A simple reparameterization to accelerate training of deep neural networks.
\newblock \emph{Advances in neural information processing systems}, 29, 2016.

\bibitem[Schaeffer et~al.(2024)Schaeffer, Lecomte, Pai, Carranza, Isik, Unell, Khona, Yerxa, LeCun, Chung, et~al.]{schaeffer2024towards}
Schaeffer, R., Lecomte, V., Pai, D.~B., Carranza, A., Isik, B., Unell, A., Khona, M., Yerxa, T., LeCun, Y., Chung, S., et~al.
\newblock Towards an improved understanding and utilization of maximum manifold capacity representations.
\newblock \emph{arXiv preprint arXiv:2406.09366}, 2024.

\bibitem[Tong et~al.(2022)Tong, Dai, Chen, Li, Li, Yi, LeCun, and Ma]{tong2022unsupervised}
Tong, S., Dai, X., Chen, Y., Li, M., Li, Z., Yi, B., LeCun, Y., and Ma, Y.
\newblock Unsupervised learning of structured representations via closed-loop transcription.
\newblock \emph{arXiv preprint arXiv:2210.16782}, 2022.

\bibitem[Touvron et~al.(2021)Touvron, Cord, Douze, Massa, Sablayrolles, and J{\'e}gou]{touvron2021training}
Touvron, H., Cord, M., Douze, M., Massa, F., Sablayrolles, A., and J{\'e}gou, H.
\newblock Training data-efficient image transformers \& distillation through attention.
\newblock In \emph{International conference on machine learning}, pp.\  10347--10357. PMLR, 2021.

\bibitem[Wang et~al.(2023)Wang, Girdhar, Yu, and Misra]{wang2023cut}
Wang, X., Girdhar, R., Yu, S.~X., and Misra, I.
\newblock Cut and learn for unsupervised object detection and instance segmentation.
\newblock In \emph{Proceedings of the IEEE/CVF Conference on Computer Vision and Pattern Recognition}, pp.\  3124--3134, 2023.

\bibitem[Wei et~al.(2024)Wei, Chen, Jin, Ma, Liu, Ling, Wang, Chen, and Zheng]{wei2024stronger}
Wei, Z., Chen, L., Jin, Y., Ma, X., Liu, T., Ling, P., Wang, B., Chen, H., and Zheng, J.
\newblock Stronger fewer \& superior: Harnessing vision foundation models for domain generalized semantic segmentation.
\newblock In \emph{Proceedings of the IEEE/CVF Conference on Computer Vision and Pattern Recognition}, pp.\  28619--28630, 2024.

\bibitem[Wu et~al.(2018)Wu, Xiong, Yu, and Lin]{wu2018unsupervised}
Wu, Z., Xiong, Y., Yu, S.~X., and Lin, D.
\newblock Unsupervised feature learning via non-parametric instance discrimination.
\newblock In \emph{Proceedings of the IEEE conference on computer vision and pattern recognition}, pp.\  3733--3742, 2018.

\bibitem[Wu et~al.(2024)Wu, Ding, Lu, Pai, Zhang, Wang, Yu, Ma, and Haeffele]{wu2024token}
Wu, Z., Ding, T., Lu, Y., Pai, D., Zhang, J., Wang, W., Yu, Y., Ma, Y., and Haeffele, B.~D.
\newblock Token statistics transformer: Linear-time attention via variational rate reduction.
\newblock \emph{arXiv preprint arXiv:2412.17810}, 2024.

\bibitem[Yang et~al.(2024)Yang, Li, Pai, Zhou, Ma, Yu, and Xie]{yang2024scaling}
Yang, J., Li, X., Pai, D., Zhou, Y., Ma, Y., Yu, Y., and Xie, C.
\newblock Scaling white-box transformers for vision.
\newblock \emph{arXiv preprint arXiv:2405.20299}, 2024.

\bibitem[Yerxa et~al.(2023)Yerxa, Kuang, Simoncelli, and Chung]{yerxa2023learning}
Yerxa, T., Kuang, Y., Simoncelli, E., and Chung, S.
\newblock Learning efficient coding of natural images with maximum manifold capacity representations.
\newblock \emph{Advances in Neural Information Processing Systems}, 36:\penalty0 24103--24128, 2023.

\bibitem[Yu et~al.(2020)Yu, Chan, You, Song, and Ma]{yu2020learning}
Yu, Y., Chan, K. H.~R., You, C., Song, C., and Ma, Y.
\newblock Learning diverse and discriminative representations via the principle of maximal coding rate reduction.
\newblock \emph{Advances in neural information processing systems}, 33:\penalty0 9422--9434, 2020.

\bibitem[Yu et~al.(2023)Yu, Buchanan, Pai, Chu, Wu, Tong, Haeffele, and Ma]{yu2023white}
Yu, Y., Buchanan, S., Pai, D., Chu, T., Wu, Z., Tong, S., Haeffele, B., and Ma, Y.
\newblock White-box transformers via sparse rate reduction.
\newblock \emph{Advances in Neural Information Processing Systems}, 36:\penalty0 9422--9457, 2023.

\bibitem[Zhou et~al.(2017)Zhou, Zhao, Puig, Fidler, Barriuso, and Torralba]{zhou2017scene}
Zhou, B., Zhao, H., Puig, X., Fidler, S., Barriuso, A., and Torralba, A.
\newblock Scene parsing through ade20k dataset.
\newblock In \emph{Proceedings of the IEEE conference on computer vision and pattern recognition}, pp.\  633--641, 2017.

\bibitem[Zhou et~al.(2021)Zhou, Wei, Wang, Shen, Xie, Yuille, and Kong]{zhou2021ibot}
Zhou, J., Wei, C., Wang, H., Shen, W., Xie, C., Yuille, A., and Kong, T.
\newblock ibot: Image bert pre-training with online tokenizer.
\newblock \emph{arXiv preprint arXiv:2111.07832}, 2021.

\end{thebibliography}
\bibliographystyle{icml2025}

\newpage
\appendix
\onecolumn
\section{Formal Description of Local and Global Views} \label{sec:views_formal}

Each local view, say \(v_{\localview}\) acts as follows, given an input image \(\vX\) of shape \((C, H, W)\). First, for a hyperparameter \(p_{\loc} \in [0, 1]\) it crops a rectangular component from \(\vX\) of shape \((C, H_{\localview}, W_{\localview})\), where \(H_{\localview}\) and \(W_{\localview}\) are chosen such that \(H_{\localview}W_{\localview} = p_{\loc}HW\), i.e., the crop is a fraction \(p_{\loc}\) of the whole image. Then the component is resized to shape \((C, S_{\loc}, S_{\loc})\), where \(S_{\loc}\) is a hyperparameter, and then divided into \(N_{\loc} := S_{\loc}^{2}/P^{2}\) square patches of shape \((C, P, P)\), where the patch size \(P\) is a hyperparameter. Each patch is unrolled into a vector of length \(D := CP^{2}\), and the \(N_{\loc}\) unrolled vectors are placed in raster order to get the output \(\vX_{\localview} \in \R^{N_{\loc} \times D}\). Each global view \(v_{\globalview}\) acts the same as a local view, except that the corresponding hyperparameters \(p_{\glo}, S_{\glo}\) are larger than their local counterparts \(p_{\loc}, S_{\loc}\) (hence also \(N_{\glo}\) vs.~\(N_{\loc}\)), while the patch size \(P\) (hence dimension \(D\)) remains the same.\footnote{Of course, we also need the patch size \(P\) to divide both the image sizes \(S_{\loc}\) and \(S_{\glo}\).}

We use these local and global views for training. For evaluation or inference, we do a similar procedure: given \(\vX\) of shape \((C, H, W)\), we resize \(\vX\) proportionally so that its \textit{shorter} edge is length \(L_{\evaluation}\), then take a square crop from the center of shape \((C, S_{\evaluation}, S_{\evaluation})\). This sequence is divided into \(N_{\evaluation} := S_{\evaluation}^{2}/P^{2}\) square patches of length \((C, P, P)\); each patch is unrolled into a vector of length \(D := CP^{2}\), and the \(N_{\evaluation}\) unrolled vectors are placed in raster order to get the output \(\vX_{\evalview} \in \R^{N_{\evaluation} \times D}\).

\section{Complex Interactions in \dino{} and Their Removal} \label{sub:dino_complexity}

We wish to showcase a finer point about why the \dino{} pipeline is so unstable. Notice that 
\begin{align}
    \CE(\vp, \vq) 
    &= -\sum_{i = 1}^{m}p_{i}\log q_{i} \\ 
    &= \sum_{i = 1}^{m}p_{i}\log(p_{i}/q_{i}) - \sum_{i = 1}^{m}p_{i}\log p_{i} \\ 
    &= \KL(\vp, \vq) + H(\vp)
\end{align}
where \(\KL\) is the KL divergence, and \(H\) is the entropy of a probability distribution. Therefore,
\begin{equation}
    d_{\CE}(\vp, \vq) = \underbrace{\frac{\KL(\vp, \vq) + \KL(\vq, \vp)}{2}}_{= d_{\JS}(\vp, \vq)} + \frac{1}{2}(H(\vp) + H(\vq)).
\end{equation}
The first term \(d_{\JS}(\vp, \vq)\) is the Jensen-Shannon divergence, which encourages \(\vp = \vq\). The second term encourages the entropy of \(\vp\) and \(\vq\) to be low, namely closer to one-hot vectors. 

Now consider the \dino{} objective:
\begin{align}
    &\cL_{\dino}(\theta_{\student}, \theta_{\teacher}, \eta_{\student}^{\dino}, \eta_{\teacher}^{\dino}, \vmu) \\
    &= \Ex[d_{\CE}(\vp_{\crop}^{\cls}(\theta_{\student}, \eta_{\student}^{\dino}), \vp_{\globalview}^{\cls}(\theta_{\teacher}, \eta_{\teacher}^{\dino}, \vmu))] \\ 
    &= \Ex\rs{d_{\JS}(\vp_{\crop}^{\cls}(\theta_{\student}, \eta_{\student}^{\dino}), \vp_{\globalview}^{\cls}(\theta_{\teacher}, \eta_{\teacher}^{\dino}, \vmu)) + \frac{H(\vp_{\crop}^{\cls}(\theta_{\student}, \eta_{\student}^{\dino})) + H(\vp_{\globalview}^{\cls}(\theta_{\teacher}, \eta_{\teacher}^{\dino}, \vmu))}{2}}
\end{align}
Suppose that, for example, \(h_{\eta_{\student}^{\dino}}\) and \(h_{\eta_{\teacher}^{\dino}}\) had ranges as a multiple of the all-ones vector, and \(\vmu\) were a constant multiple of the ones vector. Then the first term in the loss would be minimized, but the second term would become as large as possible (since both \(\vp^{\cls}\) would be just \(\frac{1}{m}\vone_{m}\), i.e., probability vectors corresponding to the uniform distribution), so this would not be the optimal solution in general. This implies that the learned \(h_{\eta_{\student}^{\dino}}\) and \(h_{\eta_{\teacher}^{\dino}}\) in general would not both be degenerate. This enables the tradeoff between the EMA parameter \(\lambda\) and the temperature parameter \(\tau\) which enables non-collapse. If the objective just involved the JS divergence and not the entropy term, or else had \(h_{\eta_{\student}^{\dino}}\) be degenerate (manually set and frozen, for instance), or else didn't have a carefully set tradeoff between \(\lambda\) and \(\tau\), then the model would collapse. However, \(\simdino\) removes all of this complexity and replaces it with an explicit coding-rate-type term.

\section{Theory for Hyperparameter Scaling}\label{sec:theory}

Let \(d, n\) be positive integers. Our main theorem is the following.

\begin{theorem}[Scale of \(\nabla R_{\eps}\)]\label{thm:optimizer_R_eps}
    We have 
    \begin{equation}\label{eq:max_R_constrained}
        \max_{\substack{\vZ \in \R^{d \times n} \\ \norm{\vZ_{i}}_{2} = 1\ \forall i}}\norm*{\nabla_{\vZ}R_{\eps}\rp{\frac{\vZ\vZ^{\top}}{n}}}_{F} \leq \frac{\sqrt{d\min\{d, n\}/n}}{4\eps}
    \end{equation}
\end{theorem}
\begin{proof}
    Let \(\alpha := d/(n\eps^{2})\) and let \(f \colon \R^{d \times n} \to \R\) be defined by
    \begin{equation}
        f(\vZ) := \logdet(\vI + \alpha \vZ\vZ^{\top}),
    \end{equation}
    i.e., \(f(\vZ) = 2R_{\eps}(\vZ\vZ^{\top}/n)\). Now, let \(r := \min\{d, n\}\). For any matrix \(\vM\), let \(\sigma_{i}(\vM)\) be its \(i^{\mathrm{th}}\) largest singular value, for \(i = 1, \dots, d\). First, note that since \(\norm{\vZ_{i}}_{2} = 1\) for all \(i\), it holds 
    \begin{equation}
        \sum_{i = 1}^{r}\sigma_{i}(\vZ)^{2} = \sum_{i = 1}^{d}\sigma_{i}(\vZ)^{2} = \sum_{i = 1}^{d}\sigma_{i}(\vZ\vZ^{\top}) = \tr(\vZ\vZ^{\top}) = \sum_{i = 1}^{d}(\vZ\vZ^{\top})_{ii} = \sum_{i = 1}^{d}\underbrace{\norm{\vZ_{i}}_{2}}_{= 1} = d.
    \end{equation}
    Now, we can simplify the gradient. It holds
    \begin{equation}
        \nabla f(\vZ) = \alpha (\vI + \alpha \vZ\vZ^{\top})^{-1}\vZ.
    \end{equation}
    Thus, it holds that 
    \begin{align}
        \norm{\nabla f(\vZ)}_{F}^{2}
        &= \tr([\nabla f(\vZ)]^{\top}[\nabla f(\vZ)]) \\ 
        &= \alpha^{2}\tr(\vZ^{\top}(\vI + \alpha \vZ\vZ^{\top})^{-2}\vZ).
    \end{align}
    Using that the trace is the sum of singular values, it holds by taking the SVD of \(\vZ\) that
    \begin{align}
        \tr(\vZ^{\top}(\vI + \alpha \vZ\vZ^{\top})^{-2}\vZ)
        &= \sum_{i = 1}^{r}\sigma_{i}(\vZ^{\top}(\vI + \alpha \vZ\vZ^{\top})^{-2}\vZ) \\
        &= \sum_{i = 1}^{r}\frac{\sigma_{i}(\vZ)^{2}}{[1 + \alpha \sigma_{i}(\vZ)^{2}]^{2}}.
    \end{align}
    In this case we directly optimize over the singular values, obtaining the problem 
    \begin{equation}
        \max_{\substack{\vZ \in \R^{d \times n} \\ \norm{\vZ_{i}}_{2} = 1\ \forall i}}\norm{\nabla f(\vZ)}_{F} \leq \max_{\substack{\vx \in \R^{r} \\ x_{i} \geq 0\ \forall i \\ \sum_{i = 1}^{r}x_{i} = d}}\sum_{i = 1}^{r}\frac{x_{i}}{(1 + \alpha x_{i})^{2}}.
    \end{equation}
    The function \(t \mapsto \frac{t}{(1 + \alpha t)^{2}}\) on \([0, \infty)\) has a global maximum at \(t=\frac{1}{\alpha}\), and the value is \(\frac{1}{4\alpha}\). Therefore it follows that
    \begin{equation}
        \max_{\substack{\vx \in \R^{r} \\ x_{i} \geq 0\ \forall i \\ \sum_{i = 1}^{r}x_{i} = d}}\sum_{i = 1}^{r}\frac{x_{i}}{(1 + \alpha x_{i})^{2}} \leq \max_{\substack{\vx \in \R^{r} \\ x_{i} \geq 0\ \forall i}}\sum_{i = 1}^{r}\frac{x_{i}}{(1 + \alpha x_{i})^{2}} = \frac{r}{4\alpha}.
    \end{equation}
    Unpacking this notation, we obtain 
    \begin{equation}
        \norm{\nabla f(\vZ)}_{F}^{2} \leq \alpha^{2}\cdot\frac{r}{4\alpha} = \frac{\alpha r}{4} = \frac{d\min\{d, n\}}{4n\eps^{2}}.
    \end{equation}
    Taking square roots, it holds 
    \begin{equation}
        \norm{\nabla f(\vZ)}_{F} \leq \frac{\sqrt{d\min\{d, n\}/n}}{2\eps}.
    \end{equation}
    Therefore, 
    \begin{equation}
        \norm*{\nabla_{\vZ}R_{\eps}\rp{\frac{\vZ\vZ^{\top}}{n}}}_{F} \leq \frac{1}{2}\norm{\nabla f(\vZ)}_{F} \leq \frac{\sqrt{d\min\{d, n\}/n}}{4\eps}
    \end{equation}
    as desired.
\end{proof}

\begin{remark}
    It is possible that the inequality
    \begin{equation}
        \max_{\substack{\vZ \in \R^{d \times n} \\ \norm{\vZ_{i}}_{2} = 1\ \forall i}}\norm{\nabla f(\vZ)}_{F} \leq \max_{\substack{\vx \in \R^{r} \\ x_{i} \geq 0\ \forall i \\ \sum_{i = 1}^{r}x_{i} = d}}\sum_{i = 1}^{r}\frac{x_{i}}{(1 + \alpha x_{i})^{2}}.
    \end{equation}
    is met with equality; proving this would require exhibiting a \(\vZ\) fulfilling the constraints of the first problem such that it has the prescribed singular values which solve the second problem. We do not need to do so here for the purposes of using the bound (e.g., for learning rate scaling).
\end{remark}

\begin{remark}
    While the quick-and-dirty bound
    \begin{equation}\label{eq:sv_bound_trivial_inequality}
        \max_{\substack{\vx \in \R^{r} \\ x_{i} \geq 0\ \forall i \\ \sum_{i = 1}^{r}x_{i} = d}}\sum_{i = 1}^{r}\frac{x_{i}}{(1 + \alpha x_{i})^{2}} \leq \frac{r}{4\alpha},
    \end{equation}
    by way of ignoring the constraint \(\sum_{i = 1}^{r}x_{i} = d\) seems like it could significantly loosen the bound, we do not believe this is the case. In particular, when \(1/\alpha \leq d/r\), note that setting \(x_{1} = \cdots = x_{r - 1} = 1/\alpha\) and \(x_{r} = d - (r - 1)/\alpha\) sandwiches the objective between \((r - 1)/(4\alpha)\) and \(r/(4\alpha)\), so the maximum is at least the same asymptotic order, in the very reasonable case that \(\eps\) is small enough that \(1/\alpha \leq d/r\), i.e., using the definition of \(\alpha\), such that
    \begin{equation}
        \frac{1}{\alpha} \leq \frac{d}{r} \iff \eps^{2} \leq \frac{d^{2}}{n\min\{d, n\}} \iff \eps^{2} \leq \max\bc{\frac{d}{n}, \frac{d^{2}}{n^{2}}}.
    \end{equation}
    Similar strategies should hold if we allow for an absolute constant \(c \geq 1\) such that \(1/\alpha \leq cd/r\), etc, relaxing the requirement while preserving the asymptotic order of the LHS of \eqref{eq:sv_bound_trivial_inequality}.
\end{remark}

\section{Training Pipeline Pseudocode}\label{app:pipeline_pseudocode}

In this section we provide pseudocode for the training pipelines of \simdino{} and \simdino{}v2.

\begin{algorithm}
    \caption{\simdino{} training pipeline.}\label{alg:simdino_training_pipeline}
    \footnotesize
\begin{lstlisting}[language=Python]
# fs, ft: student and teacher networks, this time outputting ONLY the cls token feature
# eps: coding rate regularization quantization hyperparameter
# gamma: coding rate regularization strength hyperparameter
# lam: teacher network EMA rate
ft.params = fs.params
for x in loader: # load a minibatch x of B samples
    xg, xl = global_views(x), local_views(x) # (B, M_glo, D, N_glo), (B, M_loc, D, N_loc)
    
    zsg, zsl = fs(xg), fs(xl) # student output (B, M_glo, d), (B, M_loc, d)
    ztg = ft(xg) # teacher output (B, M_glo, d)

    zs = cat([zsg, zsl], dim=1)  # (B, M, d) where M = M_loc + M_glo

    sq_dists = sum((zs.view(B,M,1,d) - ztg.view(B,1,M_glo,d)) ** 2, dim=3) # (B, M, M_glo)
    
    zsg_bdim = zsg.transpose(0, 1)  # (M_glo, B, d)
    covs = zsg_bdim.transpose(1, 2) @ zsg_bdim / B  # (M_glo, d, d)
    R_eps = batch_logdet(I_d.unsqueeze(0) + d/(eps**2) * covs)  # (M_glo)
    
    loss = mean(sq_dists) - gamma * mean(R_eps)
    loss.backward() # back-propagate
    
    # student and teacher updates
    update(fs) # SGD or Adam
    ft.params = lam * ft.params + (1 - 1am) * fs.params
\end{lstlisting}
\end{algorithm}

\begin{algorithm}
    \caption{\simdino{}v2 training pipeline.}\label{alg:simdinov2_training_pipeline}
    \footnotesize
\begin{lstlisting}[language=Python]
# fs, ft: student and teacher networks, this time outputting BOTH the cls token feature and patch token features
# eps: coding rate regularization quantization hyperparameter
# gamma: coding rate regularization strength hyperparameter
# lam: teacher network EMA rate
# alpha: proportion of patches that get masked
ft.params = fs.params
for x in loader: # load a minibatch x of B samples
    m = generate_mask(x, alpha)  # boolean mask (B, N)
    
    xg, xl = global_views(x), local_views(x) # (B, M_glo, D, N_glo), (B, M_loc, D, N_loc)
    
    xmg, xml = apply_mask(xg, m), apply_mask(xl, m) # (B, M_glo, N_glo), (B, M_loc, N_loc)
    
    zsg, Zsg = fs(xmg) # student on masked global views (B, M_glo, d), (B, M_glo, N, d)
    zsl, Zsl = fs(xml) # student on masked local views (B, M_loc, d), (B, M_loc, N, d)
    
    ztg, Ztg = ft(xg) # teacher output on global views (B, M_glo, d), (B, M_glo, N, d)
    
    zs = cat([zsg, zsl], dim=1)  # (B, M, d), M = M_loc + M_glo
    Zs = cat([Zsg, Zsl], dim=1)  # (B, M, N, d)

    sq_dists = sum((zs.view(B,M,1,d) - ztg.view(B,1,M_glo,d)) ** 2, dim=3) # (B, M, M_glo)
    psq_dists = mean(
        sum((Zs.view(B,M,1,N,d) - Ztg.view(B,1,M_glo,N,d)) ** 2, dim=4) # (B, M, M_glo, N)
        * m.view(B,1,1,N),  # (B, 1, 1, N)
        dim=3
    )  # (B, M, M_glo)
    
    zsg_bdim = zsg.transpose(0, 1)  # (M_glo, B, d)
    covs = zsg_bdim.transpose(-2, -1) @ zsg_bdim / B  # (M_glo, d, d)
    R_eps = batch_logdet(I_d.unsqueeze(0) + d/(eps**2) * covs)  # (M_glo)
    
    loss = (mean(sq_dists) + mean(psq_dists))/2 - gamma * mean(R_eps)
    loss.backward() # back-propagate
    
    # student and teacher updates
    update(fs) # SGD or Adam
    ft.params = lam * ft.params + (1 - lam) * fs.params
\end{lstlisting}
\end{algorithm}

\section{Implementation Details}\label{app:impl_details}

The training codes and hyperparameters for \ours{} and \ours{}v2 are derived from the released official settings in \dino{} and \dino{}v2 separately, see \cref{tab:hyperparameters} for detailed comparison. Notes that for \ours{}v2, we choose to use bfloat16 dtype in student backbone parameters and reductions for better numerical stability while other modules uses the same FSDP mixed precision settings from \dino{}v2.

\begin{table}[h]
\renewcommand{\arraystretch}{1.2}
\centering
\begin{tabular}{|l|l|c|c|c|c|}
\hline
\multicolumn{2}{|c|}{\textbf{Hyperparameter}} & \textbf{\ours{}v2} & \textbf{\dino{}v2} & \textbf{\ours{}} & \textbf{\dino{}} \\ \hline
\multirow{6}*{Model} & Patch size & \multicolumn{4}{|c|}{16} \\ \cline{2-6}
~& Register tokens & \multicolumn{2}{|c|}{4} & \multicolumn{2}{|c|}{0}\\ \cline{2-6}
~& Pos-embedding anti-alias & \multicolumn{2}{|c|}{True} & \multicolumn{2}{|c|}{False} \\ \cline{2-6}
~& Init layer scale & 0.1 &1e-5 & \multicolumn{2}{|c|}{-} \\ \cline{2-6}
~& Drop path rate & \multicolumn{2}{|c|}{0.3} & \multicolumn{2}{|c|}{0.1} \\ \cline{2-6}
~& Weight normalize last layer& \color{red} removed & True & \color{red} removed& True \\ \cline{2-6}
~& Output prototypes K & \color{red} removed & 65536 & \color{red} removed& 65536 \\ \hline
\multirow{8}*{Pipeline} &Init EMA momentum & 0.9 &0.992 & \multicolumn{2}{|c|}{0.996} \\ \cline{2-6}
~& Centering temperature& \color{red} removed &0.07 & \color{red} removed&0.07  \\ \cline{2-6}
~& Warm-up  temperature & \color{red} removed & 0.04 & \color{red} removed & 0.04  \\ \cline{2-6}
~& Warm-up temperature epochs & \color{red} removed & 30 & \color{red} removed& 30  \\ \cline{2-6}
~ & iBOT sample prob. & \multicolumn{2}{|c|}{0.5}&\multicolumn{2}{|c|}{-} \\ \cline{2-6}
~& iBOT mask ratio & \multicolumn{2}{|c|}{0.1-0.5}&\multicolumn{2}{|c|}{-}   \\ \cline{2-6}
~& iBOT head tying &  \multicolumn{2}{|c|}{False} &  \multicolumn{2}{|c|}{-} \\ \cline{2-6}
~& Koleo loss weight& \color{red} removed & 0.1 & \multicolumn{2}{|c|}{-}\\ \hline
\multirow{5}*{Data} &Global crops scale & \multicolumn{4}{|c|}{ 0.4 - 1} \\ \cline{2-6}
~& Local crops scale & \multicolumn{4}{|c|}{ 0.05 - 0.4} \\ \cline{2-6}
~& Local crops number & \multicolumn{4}{|c|}{10} \\ \cline{2-6}
~& Global crops size & \multicolumn{4}{|c|}{224} \\ \cline{2-6}
~& Local crops size & \multicolumn{4}{|c|}{96} \\ \hline
\multirow{9}*{Optim.} &Batch size & \multicolumn{2}{|c|}{128x8} & \multicolumn{2}{|c|}{64x8} \\ \cline{2-6}
~& Epochs & \multicolumn{4}{|c|}{100} \\ \cline{2-6}
~& Warm-up epochs & \multicolumn{4}{|c|}{10} \\ \cline{2-6}
~&Freeze last layer epochs&\color{red} removed & 1 &\color{red} removed  & 1 \\ \cline{2-6}
~& Learning rate & \multicolumn{2}{|c|}{0.004} & \multicolumn{2}{|c|}{0.002} \\ \cline{2-6}
~& Layerwise lr decay & \multicolumn{2}{|c|}{0.9} & \multicolumn{2}{|c|}{-} \\  \cline{2-6}
~& Weight decay &  \multicolumn{4}{|c|}{0.04} \\ \cline{2-6}
~& Weight decay end &  \multicolumn{4}{|c|}{0.4} \\ \cline{2-6}
~& Gradient clip & \multicolumn{2}{|c|}{3.0} & \multicolumn{2}{|c|}{0.3} \\  \hline
\end{tabular}
\vspace{0.1in}
\caption{Training hyperparameters used in the experiments}
\label{tab:hyperparameters}
\end{table}

\section{Additional Experiments}\label{app:more_expts}
\subsection{Ablations on Stability of \dino{} Training}
In~\Cref{tab:dino_hparam_sensitivity}, we study the optimization behavior and stability of \dino{} by varying hyperparameters that are specific to its pipeline. Specifically, we select teacher momentum, whether to apply normalization for the last layer, and teacher temperature. We vary each of them and study their impact on \dino{} training. As shown in~\Cref{tab:dino_hparam_sensitivity}, moderate adjustments for each component leads to divergence (during early training stages). These results suggest DINO training can be highly unstable and requires careful tuning efforts.

\subsection{Ablation Studies on Batch Sizes}
We vary the batch sizes when training ViT-S using \ours{} and report the results in~\Cref{tab:bs_ablation}. We observe that \simdino{} is robust to the choice of batch sizes and can converge to reasonably good performance with a smaller batch size of 256.

\subsection{Experiments on Longer Training}
More training epochs in SSL typically lead to better performance. We provide the performance of \simdino{} when doubling the number of epochs in~\Cref{tab:epochs_ablation}.
Clearly, these results show the efficacy of longer training for \simdino{}.

\subsection{\dino{} without Self-Distillation} \label{sub:dino_without_teacherstudent}
Due to the explicit coding rate regularization, it is possible to train \simdino{} without self-distillation. To validate this, we train ViT-S models on ImageNet-1k by setting the teacher network to be the student network at each iteration, effectively removing the EMA operation. Results are presented in~\Cref{tab:imagenet_no_teacherstudent}. We can see that the original \dino{} collapses under this setup for reasons discussed in \Cref{sub:dino_complexity}, while \simdino{} is able to yield non-trivial performance. It is worth noting that compared to training with full self-distillation, this variant primarily lags behind in terms of $k$-NN performance while the gap in linear probe is significantly smaller.

\begin{table}[t!]
    \caption{\small\textbf{Sensitivity of \dino{} on selected hyperparameters.} We pick three \dino{}-specific hyperparameters (i.e. teacher momentum, last-layer head normalization, teacher temperature) of the official configuration in \citep{caron2021emerging} to study their impact. Varying each one leads to divergence in early training.} 
    \centering
    \vspace{0.25em}
    \begin{tabular}{@{}ccccc@{}}
        \toprule
        Config & Mom. & Norm. & Temp.  & $k$-NN \\
        \midrule
        official (400ep) & $0.996$ & $\checkmark$ & $0.04 \rightarrow 0.07$   & 76.1  \\
        & 0.90 & $\checkmark$ & $0.04 \rightarrow 0.07$   & NaN \\
        & 0.996 & $\times$ & $0.04 \rightarrow 0.07$ & NaN \\
        & 0.996 & $\checkmark$ & $0.07$ & NaN \\
        \bottomrule
    \end{tabular}
    \label{tab:dino_hparam_sensitivity}
\end{table}

\begin{table}[t!]
    \centering
    \begin{tabular}{lccc}
        \toprule
        Batch size & 256 & 512 & 1024 \\
        \midrule
        $k$-NN &  68.3 & 69.7 & 69.6  \\
        \bottomrule
    \end{tabular}
    \caption{\small\textbf{Effect of batch sizes}. We evaluate $k$-NN accuracy of ViT-S pretrained on ImageNet-1k for 100 epochs.}
    \label{tab:bs_ablation}
\end{table}

\begin{table}[t!]
    \centering
    \begin{tabular}{lccc}
        \toprule
        Method & Epochs & $k$-NN & Linear \\
        \midrule
        \dino{} & 100 & 72.9 & 76.3 \\
        \ours{} &  100 & 74.9 & 77.3  \\
        \midrule
        \dino{} & 200 & 73.6 & 77.1 \\
        \ours{} & 200 & 76.0 & 77.7 \\
        \bottomrule
    \end{tabular}
    \caption{\small\textbf{Effect of training epochs}. We evaluate ViT-B pretrained on ImageNet-1k for 200 epochs.}
    \label{tab:epochs_ablation}
\end{table}

\begin{table}[t!]
    \centering
    \begin{tabular}{lccccc}
        \toprule
        Method & Model & self-distillation & Epochs & $k$-NN & Linear \\
        \midrule
        DINO & ViT-S & $\times$ & 100 & -- & -- \\
        SimDINO & ViT-S & $\times$ & 100 & 58.6 &  68.0 \\
        \midrule
        \color{gray} SimDINO & \color{gray} ViT-S & \color{gray} \checkmark & \color{gray} 100 & \color{gray} 69.7 &  \color{gray} 73.6 \\
        \bottomrule
    \end{tabular}
    \vspace{0.1in}
    \caption{Performance on ImageNet-1K without self-distillation.}
    \label{tab:imagenet_no_teacherstudent}
\end{table}

\subsection{Visualization of Attention Maps}\label{app:vis}
Following \cite{oquab2023dinov2,caron2021emerging}, we provide visualizations of self-attention maps of different models for qualitative comparison. We use test images that do not appear during pretraining. More concretely, we compute and visualize the average of self-attention maps across all attention heads from the last layer in~\Cref{fig:attn_head_mean_visualization}. It is clear from the attention maps that all methods studied in our paper lead to prominent segmentation properties that emerge from vision self-supervised learning.

\begin{figure}[t!]
    \centering
    \vspace{-0.4in}
    \includegraphics[width=1.0\linewidth]{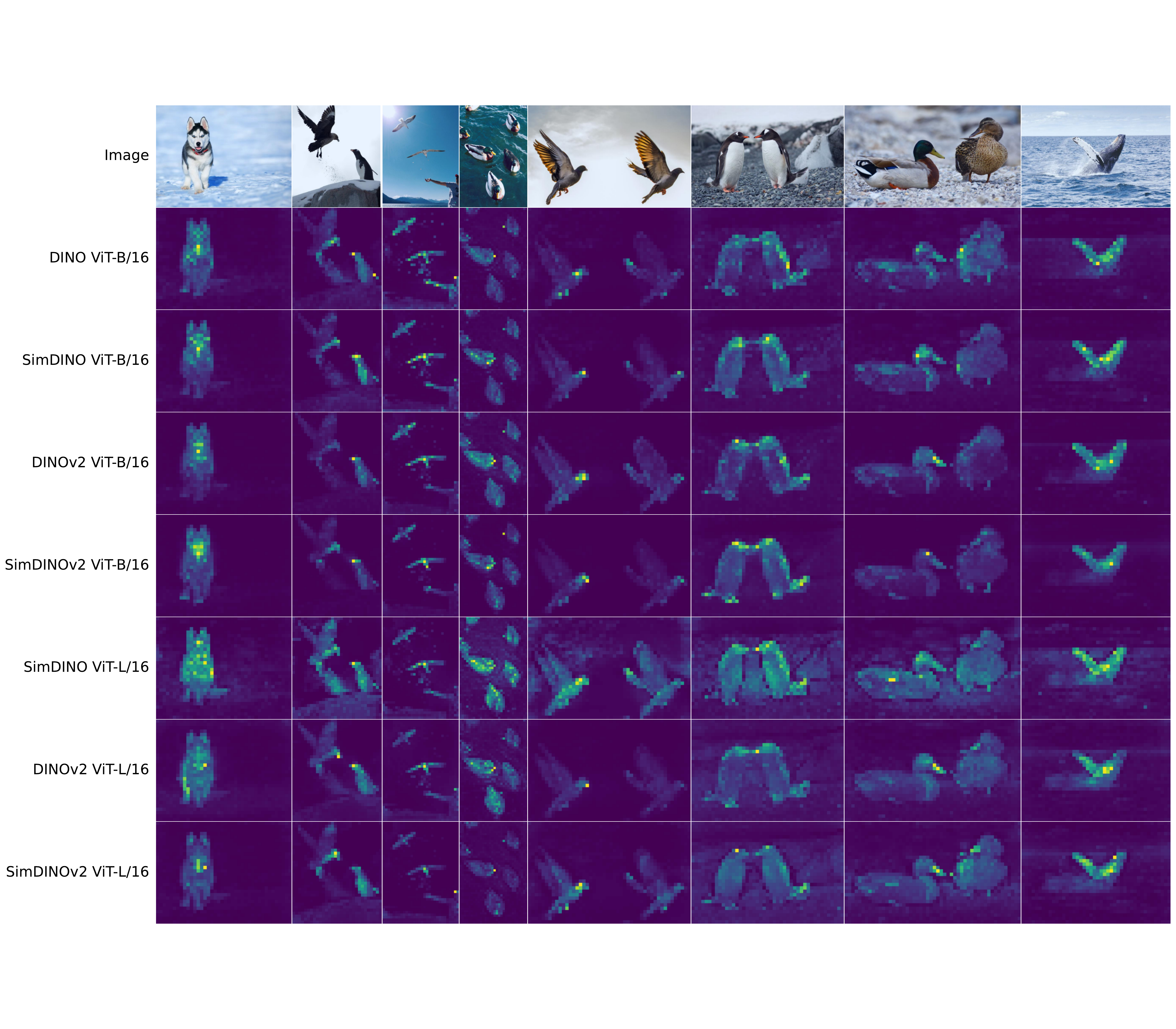}
    \vspace{-0.4in}
    \caption{\small \textbf{Visualization of average self-attention maps} obtained from both \dino{}(v2) and \simdino{}(v2) algorithms.
    }
    \label{fig:attn_head_mean_visualization}
\end{figure}

\end{document}